\documentclass[preprint,3p,times]{elsarticle}

\usepackage{ecrc}

\volume{00}

\firstpage{1}

\runauth{Ordozgoiti et al.}

\usepackage{amssymb}

\usepackage[figuresright]{rotating}

\usepackage{subfig}
\usepackage{amsfonts}
\usepackage{amsthm}
\usepackage{amsmath}
\usepackage{multirow}
\newtheorem{problem}{Problem}
\newtheorem{theorem}{Theorem}
\newtheorem{lemma}{Lemma}

\usepackage{graphicx}
\graphicspath{{images/}{./}{results/}}

\DeclareMathOperator*{\argmin}{\arg\!\min}
\DeclareMathOperator*{\argmax}{\arg\!\max}

\usepackage[utf8]{inputenc}

\makeatletter
\def\ps@pprintTitle{%
 \let\@oddhead\@empty
 \let\@evenhead\@empty
 \def\@oddfoot{}%
 \let\@evenfoot\@oddfoot}
\makeatother

\begin{document}

\begin{frontmatter}



\dochead{}

\title{Regularized Greedy Column Subset Selection} 


\author[upmcs]{ Bruno Ordozgoiti*}
\author[upmcs]{ Alberto Mozo}
\author[upmam]{ Jesús García López de Lacalle}
\address[upmcs]{Department of Computer Systems, Universidad
  Polit\'ecnica de Madrid}
\address[upmam]{Department of Applied Mathematics, Universidad
  Polit\'ecnica de Madrid}
\address{*bruno.ordozgoiti@upm.es}

\begin{abstract}
The Column Subset Selection Problem provides a natural framework for
unsupervised feature selection. Despite being a hard combinatorial
optimization problem, there exist efficient algorithms that provide
good approximations. The drawback of the problem formulation is that
it incorporates no form of regularization, and is therefore very
sensitive to noise when presented with scarce data. In this paper we
propose a regularized formulation of this problem, and derive a
correct greedy algorithm that is similar in efficiency to existing
greedy methods for the unregularized problem. We study its adequacy
for feature selection and propose suitable formulations. Additionally,
we derive a lower bound for the error of the proposed problems.
Through various numerical experiments on real and synthetic data, we
demonstrate the significantly increased robustness and stability of
our method, as well as the improved conditioning of its output, all
while remaining efficient for practical use. 
\end{abstract}


\end{frontmatter}



\section{Introduction}
\label{sec:intro}

Among dimensionality reduction methods, those that select features instead of transforming them are convenient when preserving the semantic meaning of the variables is necessary. This, together with the
abundance of large, unlabelled data sets, motivates the study of efficient unsupervised feature selection algorithms. Column Subset Selection is a combinatorial optimization problem that translates naturally to this purpose. It can be formulated as follows.

\begin{problem}{Column Subset Selection Problem (CSSP).}\label{def:cssp}
Given a matrix $A \in \mathbb{R}^{m\times n}$ and a positive integer
$k$ smaller than the rank of $A$, let $\mathcal{A}_k$ denote the set of
$m\times k$ matrices comprised of $k$
columns of $A$. Find
\begin{equation}
\label{eq:cssp}
\underset{C \in \mathcal{A}_k}\argmin\|A-CC^+A
\|_F 
\end{equation}
where $C^+$ is the Moore-Penrose pseudoinverse of
$C$.
\end{problem}

This objective is akin to the singular value decomposition, but we constrain the basis vectors to be a selection of the columns originally present in our data matrix. It attains its minimum at the matrix
$C$, composed by a subset of $k$ columns of $A$, with which we can best approximate the rest of the columns of $A$.

Problem \ref{eq:cssp} is a hard combinatorial problem, known to be
UG-hard \cite{ccivril2014column}. Only very recently did a claimed
proof of NP-completeness of the decision version appear \cite{shitov2017column}. Therefore,
for practical applications it is interesting to find efficient approximation algorithms. One example is a greedy algorithm, which in addition to having approximation guarantees
\cite{altschuler2016greedy}, works very well in practice and can be implemented very efficiently \cite{farahat2011efficient}.

An important drawback of problem \ref{eq:cssp} is that the approximation $CC^+A$ is unregularized, and therefore the coefficient matrix $C^+A$ can grow unbounded without incurring any penalty. This means
that if the matrix $A$ contains contingent quirks (e.g. noisy measurements), any algorithm might yield spuriously expressive column subsets, which might later perform badly on new data.

Our purpose in this paper is to endow the CSSP with a regularization penalty so that column subsets leading to parsimonious approximations are favored. In addition, we want
to derive an algorithm that is comparably efficient to the existing greedy method for the CSSP, and correctly optimizes each subproblem corresponding to the greedy approximation.

The main contributions of this paper can be summarized as follows:
\begin{itemize}
\item We propose a regularized formulation of the column subset selection problem.
\item We show that assuming certain conditions are met, there exists
  an algorithm that solves the greedy optimization objective at
  iteration $t$ in $O(\min\{np,
  n^2\})$ time complexity, where $p=\max\{m,t\}$.
In addition, we show that a procedure can be developed to ensure that said conditions are met as we build the set $S$ incrementally, thus allowing us to derive an efficient, correct greedy algorithm for
problem.
\item We discuss how this approach, if adopted in a naive fashion, can be inadequate for feature selection and propose an alternative objective, as well as an algorithm, to overcome this drawback.
\item We offer a lower bound for the error of the proposed problems, which can serve to inform a stopping criterion.
\end{itemize}

\section{Related work}
\label{sec:related}
Many works have explored the problem of unsupervised feature selection in the past decades, some focusing on the column subset selection problem and others following different approaches. Here we
offer a brief overview.

\subsection{Locality and cluster structure preservation}
Many proposals try to find algorithms that preserve the manifold structure of the data, choosing features that preserve local affinity between data instances.
A seminal work in this field proposes to build a k-nearest neighbor matrix \cite{he2005laplacian}. Then, features preserving the structure of said matrix are kept. The feature scores can be computed
as a function of the Laplacian of the graph encoded by the nearest neighbor matrix. In \cite{zhao2007spectral} this approach is generalized in a framework that applies to both supervised and
unsupervised feature selection. In the    
supervised case, the similarity between instances belonging to the same class is represented with a per-class constant value. Furthermore, the features can be ranked based on their similarity to
the eigenvectors of the normalized Laplacian of the similarity graph. The notion of choosing features by examining their ability to preserve the local connectivity or manifold structure of the data
is explored further in various works, considering a penalty term to exploit unlabelled data in an SVM \cite{xu2010discriminative}, performing sparse regression on the spectral embedding of the data
\cite{cai2010unsupervised},  \cite{zhao2010efficient} or optimizing a discriminator matrix with an $\ell{2,1}$-norm penalty \cite{yang2011l2} \cite{li2012unsupervised}. Many of these approaches were
shown to be realizations of a general framework, as well as ineffective to detect redundancies, by Zhao et al. \cite{zhao2013similarity}.

  The  use of the $\ell{2,1}$-norm penalty on a coefficient matrix is convenient for feature selection, as it promotes row sparsity, and can be found in several other works \cite{hou2011feature},
  \cite{he20122}, \cite{qian2013robust}, \cite{hou2014joint}, \cite{wang2015embedded}, \cite{du2015unsupervised}. Some recent works propose the use of a regularized coefficient matrix to choose features
  that  can reconstruct the entire data set well in a manner that is similar to the problem we tackle in this paper \cite{du2015unsupervised}, \cite{wang2015unsupervised}.

  \subsection{Column subset selection}
  The column subset selection problem (CSSP) has also received significant attention over the last few years. Its origins can be traced back to factorizations with column pivoting to
  isolate well-conditioned column subsets \cite{chan1987rank}. More
  recently, several methods have been proposed to approximate a matrix
  based on a subset of its entries, generally building upon the ideas
  introduced in the seminal
  work of Frieze and Vempala \cite{frieze2004fast}. One such method is the CUR matrix decomposition \cite{mahoney2009cur}. Other methods that explicitly target the CSSP with approximation guarantees
  have been proposed \cite{boutsidis2009improved} \cite{boutsidis2014near}. An adequate alternative for practical applications is the greedy algorithm, which generally provides good subset choices and can
  be implemented very efficiently \cite{farahat2011efficient}. It has been shown that this algorithm also enjoys approximation guarantees \cite{civril2012column} \cite{altschuler2016greedy}. Recently, an
  efficient local-search method was shown empirically to outperform other existing approaches in terms of the objective function \cite{ordozgoiti2016fast}. A lower bound for the recovered matrix norm was
  proved in subsequent work \cite{ordozgoiti2017iterative}.

  Most existing proposals for the CSSP attempt to solve the optimization problem directly. The inconvenient of this approach is that the resulting approximation model is unregularized, and is
  therefore sensitive to irrelevant idiosyncrasies of the input data, like outliers and noise, especially when data are scarce. In this paper we propose the addition of a regularization term to the
  problem formulation, making the solutions less sensitive to undesirable peculiarities of the available data. In addition, we show that the new problem can still be approximately optimized very
  efficiently in a greedy fashion, which makes our method well-suited for practical use.

\section{Greedy column subset selection}
\label{sec:greedy_css}
\paragraph{Notation}
\begin{itemize}
\item $A_{i:}$ is the $i$-th row and $A_{:i}$ the $i$-th column of $A$. 
\item $A_{ij}$ is the entry in the $i$-th row and $j$-th column of matrix $A$.
\item $\mathcal P(S)$: the power set of set $S$
\item $|S|$ is the cardinality of set $S$
\item Given a matrix $A \in \mathbb R^{m \times n}$ and a set $S \subset \{1, \dots, n\}$, $A_S$ is the $m \times k$ matrix comprised of the columns of $A$ whose indices are in the set $S$.
\item If $x_i$ is a vector, then $x_{ii}$ is its $i$-th element.
\item $[n]$ denotes the subset of $\mathbb N$ defined as   $\mathbb N \cap [1,n]$, 
\item Given vectors $x,y$ of the same dimension, $x\circ y$ denotes the element-wise product of $x$ and $y$.
\item Given two matrices $A$ and $B$, $(A~B)$ is the matrix resulting
  from appending the columns of $B$ to $A$. E.g., if $A \in
  \mathbb{R}^{m\times n}, B  \in \mathbb{R}^{m\times  k}$, then $(A~B)
  \in \mathbb{R}^{m\times n+k}$ and consists of the columns of
  both matrices. The notation $ \left ( \begin{array}{c}  A
    \\ B  \end{array} \right )$, is analogous, but denotes row-wise, instead of column-wise,
  concatenation.

\end{itemize}

The greedy maximization of the CSSP objective is based on the
following observation. Let $A \in \mathbb R^{m\times n}$ and let $C$
be a matrix composed by a proper subset of the columns of $A$. Let
$\tilde C$ denote the matrix that results from adding an additional
column of $A$ to $C$, that is $\tilde C = (C~A_{:w})$ for
some column index $w$. Then
\begin{equation}
  \label{eq:greedy_fact}
A-\tilde C \tilde C^+ A = A - CC^+A - E_{:w}E_{:w}^+E
\end{equation}

where $E=A-CC^+A$. This is easily seen considering that $CC^+$ is a projection onto the space spanned by the columns of $C$, and therefore all the columns of $E$ are orthogonal to those of $CC^+A$.

Equation \ref{eq:greedy_fact} implies that we can easily perform a
greedy selection of columns, by updating $A$ at each step to obtain
the corresponding residual $E$ and choosing a column $w$ of $E$ to
minimize $\|E-E_{:w}E_{:w}^+E\|_F^2$. Furthermore, the minimizing
column can be found efficiently. In \cite{farahat2011efficient} it is
shown that 
\begin{equation}
  \label{eq:greedy_choice}
\argmin_i \|E-E_{:w}E_{:w}^+E\|_F^2 = \argmax_i \frac{\|G_{:i}\|_2^2}{G_{ii}}
\end{equation}
where $G=E^TE$. In addition, the values of $\|G_{:i}\|_2^2$ and $G_{ii}$ for all $i$ can be updated efficiently every time we incorporate a new column to our subset.





\subsection{Regularized greedy column subset selection}
In this section we propose a modification of the CSSP objective to enable regularization. 
The approach is as follows. If we observe that $CC^+A = C(C^TC)^{-1}C^TA$ (whenever the inverse exists), we can consider Tikhonov regularization for linear regression (ridge regression) and
reformulate the CSSP as follows:

\begin{problem}\label{def:rcss}
Given a matrix $A \in \mathbb{R}^{m\times n}$ and a positive integer
$k \leq n$ and $\lambda \in \mathbb R$, let $\mathcal{A}_k$ denote the set of
$m\times k$ matrices comprised of $k$
columns of $A$. Find
\begin{equation}
\label{eq:reg_cssp}
\underset{C \in \mathcal{A}_k}\argmin\|A-C(C^TC+\lambda I)^{-1}C^TA
\|_F 
\end{equation} 
\end{problem}


This problem is equivalent to the CSSP, but using a regularized approximation of the target matrix. By introducing the term $\lambda$, we penalize subsets that would require very large coefficient
matrices. We can therefore tune its value to trade off between goodness of fit and model complexity.

The problem of this approach is that the greedy algorithm described in
section \ref{sec:greedy_css} is no longer applicable. The reason is
that this greedy method, as described by Farahat et al. \cite{farahat2011efficient}, relies
heavily on the fact that $CC^+$ is a projection. The corresponding
term of the regularized formulation, $(C^TC+\lambda I)^{-1}C^T$, ceases to be a projection whenever $\lambda > 0$. Therefore, equations \ref{eq:greedy_fact} and \ref{eq:greedy_choice} no longer hold.


Before we proceed, we will define the following function in order to simplify our notation.
For a matrix $A\in \mathbb R^{m\times n}$,  
\begin{align*}
f_A:  \mathcal P([n])\times \mathbb R & \rightarrow  \mathbb R^{m\times n}
\\ S,\lambda & \mapsto  f_A(S,\lambda)=A_{S}(A_S^TA_S + \lambda I)^{-1}A_S^TA
\end{align*}
where $\mathcal P([n])$ is the power set of $[n]$.
That is, given a set $S$ and a regularizing term $\lambda$,
$f_A(S,\lambda)$ denotes the approximation of $A$ obtained using the
columns indexed by $S$, regularized using $\lambda$.

Our goal in this paper is to derive an efficient greedy algorithm for problem \ref{def:rcss}. In essence, we need an algorithm that solves the following sub-problem:

\begin{problem}\label{def:rgcs}
Given a matrix $A \in \mathbb{R}^{m\times n}$ and a set $S \subset [n]$, find
\begin{equation}
  \label{eq:rgcs}
\argmin_i \|A-f_A(S \cup \{i\},\lambda)\|_F^2
\end{equation}
\end{problem}

Given a column subset $S$, this problem is solved by finding the best addition to said subset. Therefore, a correct greedy algorithm should solve this problem at each iteration. Obviously,
we can solve problem \ref{def:rgcs} simply by inspecting the value of the objective function for all possible choices of $i$. This method, however, involves considerable computations, 
as each candidate requires (for $|S|=t$) matrix products of complexity $O(mnt)$ and $O(nt^2)$ and a matrix inversion of complexity $O(t^3)$, in addition to the matrix subtraction and the
computation of the norm.   
In sum, each iteration of the resulting algorithm would take
$O(\max\{mn^2t,n^2t^2\})$ operations. The algorithm we propose here
solves problem \ref{def:rgcs} in $O(\min \{np, n^2\})$ time, where $p=\max\{m,t\}$.

\section{Algorithms for regularized greedy column subset selection}
\label{sec:algo}

Let us assume we are running a greedy algorithm that has so far completed $t$ iterations, i.e., we have a set $S$ of $t$ elements.
We want to solve problem \ref{def:rgcs}, that is,
$$
\argmin_i \|A-f_A(S\cup \{i\}, \lambda)\|_F^2
$$

As said before, we can of course do this by inspecting the value of this objective for all possible choices. However, this would be computationally costly. If we define $C_i=A_{S\cup \{i\}}$, in order to do this we would
need to compute the inverse of $C_i^TC_i+\lambda I$ for all choices of $i$. We can start our derivation by attempting to circumvent these matrix inversions, which we can accomplish as follows. We define

$$\hat A= \left ( \begin{array}{c}  A \\ \sqrt{\lambda} I \end{array}
\right )$$
and $\hat C_i=\hat A_{S\cup \{i\}}$.

If we observe that
$$
 C_i^T C_i +\lambda I =  \hat C_i^T \hat C_i
 $$ 
then we can take advantage of the following fact. Let us denote
$w=A_{:w}, \hat w=\hat A_{:w}$ (here we use column indices and the
corresponding vectors interchangeably). 
Since $\hat C_w = \hat A_{S \cup \{w\}}= (\tilde C\mbox{ }  \tilde w)$, then it is
well known \cite{lutkepohl1997handbook} that
$$
 (C_w^T  C_w+\lambda I)^{-1} = (\hat C_w^T  \hat C_w)^{-1} = \left ( \begin{array}{cc} (C^TC+\lambda I)^{-1} + \frac{vv^T}{\alpha_w} & \hspace{1em}-\frac{v}{\alpha_w}  \\   -\frac{v^T}{\alpha_w} &
   \hspace{1.8em} \frac{1}{\alpha_w} \end{array} \right )
$$
where $\alpha_w=\hat w^T \hat w - \hat w^T \hat C(C^TC + \lambda I)^{-1} \hat C^T\hat w$ and  $v = (C^TC+\lambda I)^{-1}C^Tw$. (Note that $C^Tw=\hat C^T \hat w$ because the extension of $w$ is multiplied by zero).


Relying on this expression for the inverse of $C_w^T C_w +\lambda I$, we can look for a solution to problem \ref{def:rgcs} that does not need to explicitly compute the value of the objective for all
choices of $i$. For brevity, let us define $A^{(t)}= f_A(S, \lambda)$
and $A^{(t+1)}= f_A(S\cup \{w\}, \lambda)$. First, observe that (for a detailed derivation, refer to the appendix)
\begin{align}
  & A^{(t+1)} =   C_w( C_w^T  C_w + \lambda I)^{-1}   C_w^TA \nonumber
  \\ = ~& A^{(t)} + \frac{1}{\alpha_w} (A^{(t)}_{:w} - w  )(A^{(t)}_{:w} - w )^T A
\label{eq:update_At}
\end{align}
Let $d_i=(A^{(t)}_{:i} - A_{:i} ) \in \mathbb{R}^m$. Then the next column choice is yielded by 
\begin{align}
&   \argmin \|A-A^{(t+1)}\|_F^2 \nonumber
\\ =& \argmin_i tr (A^TA) - tr(A^TA^{(t)}) - tr(\frac{1}{\alpha_i}A^Td_id_i^TA) \nonumber
\\&- tr((A^T)^{(t)}A) - tr(\frac{1}{\alpha_i}A^Td_id_i^TA) \nonumber
\\&+ tr((A^T)^{(t)}A^{(t)}) +  tr(\frac{1}{\alpha_i}(A^T)^{(t)}d_id_i^TA) + tr(\frac{1}{\alpha_i}A^Td_id_i^TA^{(t)}) \nonumber
\\&+ tr(\frac{1}{\alpha_i^2}A^Td_id_i^Td_id_i^TA) 
\label{eq:preargmin}
\end{align}

We define $x_i=A^Td_i \in \mathbb{R}^n$,  $\tilde x_i=(A^{(t)})^Td_i \in \mathbb{R}^n$. Dropping irrelevant constants from the previous equality,
\begin{align}  
  &\argmin \|A-A^{(t+1)}\|_F^2 \nonumber
\\=& -2tr(\frac{1}{\alpha_i}A^Td_id_i^TA) +2tr(\frac{1}{\alpha_i}(A^T)^{(t)}d_id_i^TA) + tr(\frac{1}{\alpha_i^2}A^Td_id_i^Td_id_i^TA) \nonumber
\\=& \argmin_i \frac{2}{\alpha_i}\tilde x_i^T  x_i -  \frac{2}{\alpha_i}\|x_i\|_2^2 + \frac{1}{\alpha_i^2}\|d_i\|_2^2\|x_i\|_2^2 \nonumber
\\=& \argmin_i \frac{2}{\alpha_i}\tilde x_i^Tx_i +    (\frac{-2}{\alpha_i} +  \frac{1}{\alpha_i^2}(\tilde x_{ii} - x_{ii})  )x_i^Tx_i + \lambda
\label{eq:argmin}
\end{align}

Equality \ref{eq:argmin} provides a surrogate  of problem \ref{def:rgcs}, i.e., the problem we need to solve at each iteration in order to implement a correct greedy algorithm for problem 
\ref{def:rcss}. In particular, equality \ref{eq:argmin} shows that the solution to the problem \ref{def:rgcs} can be found as a function of $x_i, \tilde x_i$ and $\alpha_i$ for $i=1,\dots, n$. 
Our concern now is to develop a procedure to find the value of these variables at each iteration without incurring too much computational cost. 

We have
\begin{align*}
\tilde x_i^Tx_i &= d_i^TA^{(t)}A^Td_i
\\x_i^Tx_i &= d_i^TAA^Td_i
\end{align*}
For notational convenience, we define the matrices $X, \tilde X$ and $D$, whose columns are the vectors $x_i, \tilde x_i$ and $d_i$, $i=1, \dots, n$ respectively.
First, observe from equality \ref{eq:update_At} that
\begin{align*}
A^{(t+1)}=A^{(t)} + \frac{1}{\alpha_w} d_wd_w^TA
\end{align*}
Hence,
\begin{align}
  A^TA^{(t+1)} = A^TA^{(t)} + \frac{1}{\alpha_w}  x_w x_w^T
  \label{eq:update_atat}
\end{align}
And
\begin{align}
(A^TA)^{(t+1)} = (A^TA)^{(t)} + \frac{1}{\alpha_w} \tilde x_w x_w^T + \frac{1}{\alpha_w} x_w\tilde x_w^T + \frac{1}{\alpha_w^2} x_wd_w^Td_wx_w^T
\label{eq:update_ata}
\end{align}
So finally, 
\begin{align}
 X^{(t+1)} & = A^TD^{(t+1)}  \nonumber
\\ & = A^TD^{(t)} + A^TA + \frac{1}{\alpha_w} x_w x_w^T -A^TA  \nonumber
\\ & = X^{(t)} + \frac{1}{\alpha_w} x_w x_w^T = X^{(0)} + \sum_{i=0}^t \left (\frac{1}{a_w}x_wx_w^T \right )^{(i)}
\label{eq:update_x}
\end{align}
And
\begin{align}
\tilde X^{(t+1)} & = (A^TD)^{(t+1)}  \nonumber
\\ & = (A^TD)^{(t)} + \frac{1}{\alpha_w} \tilde x_w x_w^T + \frac{1}{\alpha_w} x_w\tilde x_w^T + \frac{1}{\alpha_w^2} x_wd_w^Td_wx_w^T  - \frac{1}{\alpha_w} x_w  x_w^T \nonumber
\\ & = \sum_{i=0}^t \left (\frac{1}{\alpha_w} \tilde x_w x_w^T + \frac{1}{\alpha_w} x_w\tilde x_w^T + \frac{1}{\alpha_w^2} x_wd_w^Td_wx_w^T  - \frac{1}{\alpha_w} x_w  x_w^T \right )^{(i)}
\label{eq:update_xt}
\end{align}

Expressions \ref{eq:update_x} and \ref{eq:update_xt} allow us to compute the variables involved in problem \ref{eq:argmin} efficiently as more columns are greedily added to the final set.
Specifically, it is easy to see that
\begin{align}  
  & (x_i^T\tilde x_i) ^{(t+1)}  = x_i^T\tilde x_i \nonumber
\\& + \left (x_i + \frac{1}{\alpha_w}x_wx_{wi} \right )^T  \left( \tilde x_i + \frac{1}{\alpha_w}\tilde x_wx_w^T + \frac{1}{\alpha_w}x_w\tilde x_w^T + \frac{1}{\alpha_w^2}x_wx_w^T(\tilde x_{ww}-x_{ww}) - \frac{1}{\alpha_w}x_wx_w^T \right ) \nonumber
\\ & = x_i^T\tilde x_i  + \frac{1}{\alpha_w}(x_i^T\tilde x_w x_{wi} + x_i^Tx_w(\tilde x_{wi}-x_{wi})  + \tilde x_i^Tx_wx_{wi}) \nonumber
  \\ & + \frac{1}{\alpha_w^2} (x_i^Tx_wx_{wi}(\tilde x_{ww}-x_{ww}) + x_w^Tx_wx_{wi}\tilde x_{wi} + x_w^T\tilde x_w(x_{wi})^2 - x_w^Tx_w(x_{wi})^2) \nonumber
  \\ & + \frac{1}{\alpha_w^3} x_w^Tx_w(x_{wi})^2(\tilde x_{ww}-x_{ww}) \label{eq:update_xtx}
\end{align}
\begin{align}
   (x_i^Tx_i) ^{(t+1)} & = x_i^T x_i + \frac{2x_{wi}}{\alpha_w}x_i^Tx_w + \left ( \frac{x_{wi}}{\alpha_w} \right ) ^2 x_w^Tx_w 
\label{eq:update_xx}
\end{align}

All elements on the r.h.s. of equalities \ref{eq:update_xx} and \ref{eq:update_xtx} correspond to iteration $t$. The superindex has been omitted for clarity. Furthermore, in equation
\ref{eq:update_xx} we consider the values for the already chosen
indices to be irrelevant in order to simplify the expression.

The inner products involved in these updates can also be computed efficiently by making use of equalities (\ref{eq:update_x}) and (\ref{eq:update_xt}) as follows:
$$x_i^T x_w :=  \left ( x_i^{(0)} \right ) ^T( x_w)^{(t)} +  \left ( \sum_{j=0}^{t} \left (\frac{1}{a_w}x_wx_{wi} \right )^{(j)} \right ) ^T ( x_w)^{(t)}$$
$$x_i^T\tilde x_w :=  \left ( x_i^{(0)} \right ) ^T(\tilde x_w)^{(t)} +  \left ( \sum_{j=0}^{t} \left (\frac{1}{a_w}x_wx_{wi} \right )^{(j)} \right ) ^T (\tilde x_w)^{(t)}$$
$$\tilde x_i^Tx_w := \sum_{j=0}^{t} \left (\frac{1}{\alpha_w} \tilde x_w x_{wi} + \frac{1}{\alpha_w} x_w\tilde x_{wi} + \frac{\tilde x_{ww}-x_{ww}}{\alpha_w^2} x_wx_{wi}  - \frac{1}{\alpha_w} x_w  x_{wi} \right )^{(j)} (x_w)^{(t)}$$

Based on the equalities presented above, we can state and prove our main result regarding the existence of an efficient algorithm for problem \ref{def:rgcs}. The proof is constructive and
provides the necessary equalities for implementing the algorithm.

\begin{theorem}
  \label{the:main}  
Given a matrix $A\in \mathbb R^{m\times n}$, $\lambda \in \mathbb R$ and a set $S$ of cardinality $t\leq n$, let us assume the following values, as defined above, are known:
\begin{enumerate}
\item $X^{(0)}=-A^TA$
\item $x_{w(i)}^{(i)},\tilde x_{w(i)}^{(i)}, \alpha_{w(i)}^{(i)}, i=0, \dots, t-1$
\end{enumerate}
where $w(i)$ is the index of the $(i+1)$-th column added to the set $S$.
Then there exists an algorithm that solves problem \ref{def:rgcs} in $O(\min \{np, n^2\})$ time, where $p=\max\{m,t\}$.
\end{theorem}

\begin{proof}  
In order to relieve the notation, we employ $x_w^{(i)}$ to denote
$x_{w(i)}^{(i)}$.

Equality \ref{eq:argmin} gives an expression that reveals the optimum of the objective function of problem \ref{def:rgcs}. We now show that this expression can be computed in $O(\min \{np, n^2\}))$ time.

First, observe from equalities \ref{eq:update_x} and \ref{eq:update_xt} that the values of $x_w$ and $\tilde x_w$ for this iteration can be computed in $O(nt)$ time complexity as follows:
$$x_w^{(t)} := X^{(0)}_{:w} + \sum_{j=0}^{t-1} \left (\frac{1}{a_w}x_wx_{ww}^T \right )^{(j)}$$
$$\tilde x_w^{(t)} :=  \sum_{j=0}^{t-1} \left (\frac{x_{ww}}{\alpha_w} \tilde x_w  + \frac{\tilde x_{ww}}{\alpha_w} x_w + \frac{\tilde x_{ww}-x_{ww}}{\alpha_w^2} x_wx_{ww}  - \frac{x_{ww}}{\alpha_w} x_w \right )^{(j)}$$

We now define the matrices $W, \tilde W$ whose columns are respectively $x_w^{(0)}, \dots, x_w^{(t-1)}$ and  $\tilde x_w^{(0)}, \dots, \tilde x_w^{(t-1)}$. We also define the diagonal matrix
$$
B=\left (\begin{array}{ccc}  \frac{1}{\alpha_w^{(0)}} &   & 
  \\  & \ddots & 
  \\  &  &   \frac{1}{\alpha_w^{(t-1)}}
\end{array} \right )
$$
and the column-scaled matrices $W_\alpha=WB$, $\tilde W_\alpha=\tilde WB$.

The following equalities can be easily verified (from here on, we use $x_w,\tilde x_w$ to denote $x_w^{(t)},\tilde x_w^{(t)}$).
  $$
  \left ( \sum_{j=0}^{t-1} \left (\frac{1}{a_w}x_wx_{wi} \right )^{(j)} \right ) ^T  x_w^{(t)} = W_{i:}W_\alpha^Tx_w^{(t)}
  $$
  $$  
  \left ( \sum_{j=0}^{t-1} \left (\frac{1}{a_w}x_wx_{wi} \right )^{(j)} \right ) ^T \tilde x_w^{(t)} = W_{i:}W_\alpha^T\tilde x_w^{(t)}
  $$  
  
  Combined with equalities \ref{eq:update_x} and \ref{eq:update_xt}, this implies the following:
  $$(x^T_1x_w^{(t)}, \dots, x^T_nx_w^{(t)})^T = X^{(0)}x_w^{(t)} + WW_\alpha^T x_w^{(t)}$$  
  $$( x^T_1\tilde x_w^{(t)}, \dots, x^T_n\tilde x_w^{(t)})^T = X^{(0)}\tilde x_w^{(t)} +WW_\alpha^T\tilde x_w^{(t)}$$
  And similarly,  
  \begin{align*}    
  &(\tilde x^T_1x_w^{(t)}, \dots, \tilde x^T_nx_w^{(t)})^T = \tilde X^Tx_w^{(t)}
  \\ &=   W\tilde W_\alpha^Tx_w^{(t)} + \tilde WW_\alpha^Tx_w^{(t)} + \alpha^{-1}(\tilde x_w^{(t)}-x_w^{(t)}) WW_\alpha^Tx_w^{(t)} +  WW_\alpha^Tx_w^{(t)} 
  \end{align*}
  
  Since $X^{(0)}\in \mathbb R^{n\times n}$, $W,\tilde W \in \mathbb R^{n\times t}$ and $x_w^{(t)}, \tilde x_w^{(t)} \in \mathbb R^{n\times 1}$, these equalities can be computed in $O(2n^2 + 8nt)=O(n^2)$.
  If $n>m$, instead of storing $X^{(0)}$ we can explicitly compute
  $-A^TAx_w^{(t)}$ and $-A^TA\tilde x_w^{(t)}$, resulting in $O(\max\{mn,nt\})$
  time complexity for the greedy step.

  Using the variables $x^T_ix_w^{(t)}$, $x^T_i\tilde x_w^{(t)}$ and $\tilde x^T_ix_w^{(t)}$, which we have computed for $i=1, \dots, n$,
  we can now compute $\tilde x_i^T x_i$ and $x_i^Tx_i$ as shown in equalities \ref{eq:update_xtx} and \ref{eq:update_xx}. Finally, from equalities \ref{eq:update_x} and \ref{eq:update_xt} and the
  definition of $\alpha_i$, it is easily verified that  
  $$x_{ii}^{(t)} = \left (x_{ii} +
  \frac{1}{(\alpha_w)}x_{wi}^2 \right )^{(t-1)}$$  
  $$\tilde x_{ii}^{(t)} = \left ( \tilde x_{ii} + \frac{2\tilde x_{wi}x_{wi}-x_{wi}^2}{(\alpha_w)}  +  x_{wi}^2\frac{(\tilde x_{ii} - x_{ii})}{\alpha_w^2}\right )^{(t-1)}$$
  $$\alpha_i = \lambda -x_{ii}$$  
  These operations are easily seen to require $O(n)$ time complexity if computed for all $i$.
  
  Having computed all these variables, we can now compute the value of expression \ref{eq:argmin} for all $i$ in a straightforward manner, thus completing the proof.
  
\end{proof}

By theorem \ref{the:main}, if we store the value of $x_w, \tilde x_w$ and $\alpha_w$ at each iteration, we can efficiently find the best column addition for the regularized column subset 
selection formulation. This allows us to derive a greedy algorithm for problem \ref{def:rcss}.

\subsection{Algorithm}
Input: $A \in \mathbb R^{m\times n},k \in \mathbb N, k\leq n$
\begin{enumerate}
\item $X \leftarrow -A^TA$; $\alpha \leftarrow
  \lambda-diag(X)$; $x_w^{(0)} \leftarrow X_{:w}$; $\tilde x_w^{(0)} \leftarrow 0$
\item Compute $x_i^Tx_i, i=1, \dots, n$
\item Choose the first column, $w \leftarrow \argmin_i \frac{-2}{\alpha_i}-\frac{x_{ii}}{\alpha_i^2}x_i^Tx_i + \lambda$
\item $S \leftarrow \{w\}; \omega\leftarrow {\bf 0} \in \mathbb R^n$  
\item for $t=1,\dots, k-1$  
\item \hspace{1em} $x_w \leftarrow x_w^{(t-1)}$
\item \hspace{1em} $\tilde x_w \leftarrow \tilde x_w^{(t-1)}$
\item \hspace{1em} $\beta \leftarrow Xx_w + WW_\alpha^T x_w$ 
\item \hspace{1em} $\gamma \leftarrow X\tilde x_w +WW_\alpha^T\tilde x_w$
\item \hspace{1em} $\delta \leftarrow W\tilde W_\alpha^Tx_w + \tilde WW_\alpha^Tx_w + \alpha^{-1}(\tilde x_w-x_w) WW_\alpha^Tx_w +  WW_\alpha^Tx_w $
\item \hspace{1em} $\begin{aligned}[t]\omega \leftarrow \omega &+ \alpha_w^{-1}(\gamma\circ x_w + \beta\circ (\tilde x_w-x_w) + \delta \circ x_w) 
\\&+ \alpha_w^{-2}(\beta \circ x_w(\tilde x_{ww}-x_{ww}) + \beta_wx_w\circ \tilde x_w + (\gamma_w-\beta_w)x_{w}^2)
\\&+ \alpha_w^{-3}\beta_wx_w\circ x_w(\tilde x_{ww}-x_{ww})
\end{aligned}$  
\item \hspace{1em} $\psi \leftarrow \psi + \frac{2}{\alpha_w}x_w\circ \beta + \frac{\beta_w}{\alpha_w^2}x_w \circ x_w$
\item \hspace{1em} for $i=1,\dots, n$
\item \hspace{1em} \hspace{1em} $x_{ii} \leftarrow x_{ii} + \frac{1}{\alpha_w}x_{wi}^2$
\item \hspace{1em} \hspace{1em} $\tilde x_{ii} \leftarrow \tilde x_{ii} + \frac{2}{\alpha_w} \tilde x_{wi}x_{wi} + \frac{\tilde x_{ii} - x_{ii}}{\alpha_w^2} x_{wi}^2  - \frac{1}{\alpha_w} x_{wi}$
\item \hspace{1em} \hspace{1em} $\alpha_i \leftarrow \lambda -x_{ii}$
\item \hspace{1em} $w\leftarrow \argmin_i \frac{2}{\alpha_i}\omega_i +    (\frac{-2}{\alpha_i} +  \frac{1}{\alpha_i^2}(\tilde x_{ii} - x_{ii})  )\psi_i + \lambda$
\item \hspace{1em} $S \leftarrow S \cup \{w\}$
\item \hspace{1em} $x_w^{(t)} \leftarrow X_{:w} + \sum_{j=0}^{t-1} \left (\frac{1}{a_w}x_wx_{ww}^T \right )^{(j)}$  
\item \hspace{1em} $\tilde x_w^{(t)} \leftarrow  \sum_{j=0}^{t-1} \left (\frac{x_{ww}}{\alpha_w} \tilde x_w  + \frac{\tilde x_{ww}}{\alpha_w} x_w + \frac{\tilde x_{ww}-x_{ww}}{\alpha_w^2} x_wx_{ww}  - \frac{x_{ww}}{\alpha_w} x_w \right )^{(j)}$
\end{enumerate}


\subsection{An appropriate formulation for feature selection}
\label{sec:approximating}

The purpose of this algorithm is
that of selecting a few variables and approximating the rest. In the context of practical applications of feature selection, we can consider that the
chosen variables become available and do not need to be
approximated. However, the penalty of the 
regularized formulation causes the chosen variables to be imperfectly
estimated. This means that in optimizing the objective in problem \ref{def:rgcs}, we are taking
into account an error that should not be made in reality, thus providing
a potentially mistaken choice.

This can be illustrated with an example. Consider the following matrix.
$$
\arraycolsep=5pt
\left( \begin{array}{c c c c}
  1 & 0 & 0 & 1 \\  
  0 & 1 & 0 & 0 \\
  1 & 0 & 1 & 1 \\
  1 & 1 & 0 & 0 \end{array} \right)
$$
and the set $S=\{1,2\}$. The solution to problem \ref{def:rgcs} is given by adding column 4 to $S$. However, if we don't consider the error made in the approximation of the chosen columns, the best
choice is column 3. As stated above, in a feature selection setting we would not want to consider the approximation error of the chosen columns, as we can assume the corresponding variables to be
available.

For this reason, we propose the following alternative problem formulation.

\begin{problem}\label{def:ercss}
Given a matrix $A \in \mathbb{R}^{m\times n}$ and a positive integer
$k$ smaller than the rank of $A$, and defining $\bar S=[n]\backslash S$, find
\begin{equation}
\underset{S,|S|=k}\argmin\|A_{\bar S}-A_S(A_S^TA_S+\lambda I)^{-1}A_S^TA_{\bar S})\|_F^2 
\end{equation} 
\end{problem}

This is equivalent to problem \ref{def:rcss}, but it discards the error made in approximating the chosen variables. Fortunately, the algorithm proposed in section \ref{sec:algo} can be easily modified 
to greedily optimize this objective. We show how by induction on $t$.

At iteration $t$, consider equation \ref{eq:update_At}.
$$
A^{(t+1)}=A^{(t)} + \frac{1}{\alpha_w}d_wx^T
$$

If we assume that the chosen columns (those in $S$) are approximated exactly in
$A^{(t)}$, i.e. $A_S=A^{(t)}_S$, these need not be modified. Therefore, we can simply set
the corresponding positions of $x$ to zero, and these columns will not
be altered. For the next iteration, however, we are adding a new column to $S$, which we can denote $w$. Thus, we need that $A_{S\cup \{w\}}=A^{(t+1)}_{S\cup \{w\}}$.

We have $d_w=A^{(t)}_{:w}-w$. Therefore, if $x_{ww}=-\alpha_w$,
$$
A^{(t)}_{:w} +  \frac{1}{\alpha_w}d_wx_{ww} = A^{(t)}_{:w} -  d_w = w
$$

That is, by setting $x_S=0$ and $x_{ww}=-\alpha_w$ we ensure that the chosen columns are considered to be perfectly approximated.
If we set $x_{ww}^{(0)}=-\alpha_w^{(0)}$, then we ensure that the first
chosen column is approximated with no error, providing the basis for our inductive argument.

By defining $(x^*)^{(t)}$ to be equal to $x^{(t)}$ but with the previously described replacements, we can easily modify our algorithm to ensure that the column chosen at each iteration is the one that
greedily optimizes problem \ref{def:ercss}. We now detail the necessary modifications:
\begin{enumerate}
\item $x_i^T x_w^* :=  \left ( x_i^{(0)} \right ) ^T(\tilde x_w^*)^{(t)} +  \left ( \sum_{i=0}^{t-1} \left (\frac{1}{a_w}x_wx_w^T \right )^{(i)} \right ) ^T (\tilde x_w^*)^{(t)}$
\item $\tilde x_i^Tx_w := \sum_{i=0}^{t-1} \left (\frac{1}{\alpha_w} \tilde x_w x_w^T + \frac{1}{\alpha_w} x_w\tilde x_w^T + \frac{1}{\alpha_w^2} x_wd_w^Td_wx_w^T  - \frac{1}{\alpha_w} x_w  x_w^T \right )^{(i)} (x_w^*)^{(t)}$ 
\item $x_w^* := x_w$; $x_{wi}^*=0, i\in S$; $x_{ww}^*=-\alpha_w$
\item $\tilde x_w :=  \sum_{i=0}^t \left (\frac{1}{\alpha_w} \tilde x_w x_w^T + \frac{1}{\alpha_w} x_w^*\tilde x_w^T + \frac{1}{\alpha_w^2} x_w^*d_w^Td_wx_w^T  - \frac{1}{\alpha_w} x_w^*  x_w^T \right )^{(i)}$
\end{enumerate}

Equality \ref{eq:update_xtx}:
\begin{align}  
  &(x_i^T\tilde x_i) ^{(t+1)}  = x_i^T\tilde x_i \nonumber
\\ &+ \left (x_i + \frac{1}{\alpha_w}x_wx_{wi} \right )^T  \left( \tilde x_i + \frac{1}{\alpha_w}\tilde x_wx_w^T + \frac{1}{\alpha_w}x_w\tilde x_w^T + \frac{1}{\alpha_w^2}x_wx_w^T(\tilde x_{ww}-x_{ww}) - \frac{1}{\alpha_w}x_wx_w^T \right ) \nonumber
\\ & = x_i^T\tilde x_i  + \frac{1}{\alpha_w}(x_i^T\tilde x_w x_{wi} + x_i^Tx_w^*(\tilde x_{wi}-x_{wi})  + \tilde x_i^Tx_wx_{wi}) \nonumber
  \\ & + \frac{1}{\alpha_w^2} (x_i^Tx_w^*x_{wi}(\tilde x_{ww}-x_{ww}) + x_w^Tx_w^*x_{wi}\tilde x_{wi} + x_w^T\tilde x_w(x_{wi})^2 - x_w^Tx_w^*(x_{wi})^2) \nonumber
  \\ & + \frac{1}{\alpha_w^3} x_w^Tx_w^*(x_{wi})^2(\tilde x_{ww}-x_{ww}) \nonumber
\end{align}

Equality \ref{eq:update_xx}:
\begin{align}
  &(x_i^Tx_i^*)^{(t+1)}
  \\&= x_i^T x_i^* +\frac{2x_{wi}^*}{\alpha_w}x_i^T\tilde x_w + \left ( \frac{x_{wi}^*}{\alpha_w} \right )^2x_w^Tx_w - \left ( \frac{x_{wi}}{\alpha_w}x_{wS}^Tx_{wS} \right )^2 - (x_{wi} + x_{wi}\frac{x_{ww}}{\alpha_w})^2 \nonumber
\end{align}


\subsection{Choosing the value of $k$}
\label{sec:choosing}
In order to decide when to stop the algorithm, we can monitor the loss of the approximation at each iteration, i.e. $\|A-A^{(t)}\|_F^2$. Computing this quantity at each iteration, however, can be
costly. Fortunately, we can take advantage of the variables involved in the proposed algorithm to efficiently track the exact value of the loss.

By equation (\ref{eq:update_At}) we can easily see that
\begin{align*}
\|A-A^{(t)}\|_F^2 &= \|A-A^{(t-1)}-\frac{1}{\alpha_w}d_wx_w^T\|_F^2
\\&= \left (A-A^{(t-1)}-\frac{1}{\alpha_w}d_wx_w^T \right )^T \left (A-A^{(t-1)}-\frac{1}{\alpha_w}d_wx_w^T \right )
\\&= \|A-A^{(t-1)}\|_F^2 - \frac{2}{\alpha_w}x_w^Tx_w + \frac{2}{\alpha_w}\tilde x_w^Tx_w + \frac{\tilde x_{ww} - x_{ww}}{\alpha_w}x_w^Tx_w
\end{align*}

For problem \ref{def:ercss}, this equality becomes
$$
\|A-A^{(t)}\|_F^2 =\|A-A^{(t-1)}\|_F^2 - \frac{2}{\alpha_w}x_w^Tx_w^* + \frac{2}{\alpha_w}\tilde x_w^Tx_w^* + \frac{\tilde x_{ww} - x_{ww}}{\alpha_w}\left ( x_w^* \right )^Tx_w^*
$$
where $x^*$ is defined as in section \ref{sec:approximating}.

This means that we can compute the loss at each iteration as a function of its previous value and some readily available variables. The loss before the first iteration is simply the
squared norm of the data matrix, which can be computed as $tr A^TA=-tr X$.

In section \ref{sec:bound} we discuss how this quantity can be used to determine when to stop iterating.

\subsection{Lower bound for the error}
\label{sec:bound}
In the conventional formulation of the CSSP, the approximation error eventually reaches zero as we add columns to the basis subset. However, the introduction of the regularizing term bars the
approximation from being perfect. For this reason, even if we can track the loss as detailed in section \ref{sec:choosing}, it can be difficult to determine how much of an improvement can still be
made. In other words, in problem \ref{def:cssp} we can evaluate the expressive power of our column subset by checking how far the error is from zero. On the contrary, in the case of problem
\ref{def:rcss} we do not know what this ideal lower bound is.

In the case of problem \ref{def:ercss}, the loss can of course reach zero, but rather artificially (when $k=n$, we simply evaluate the error over zero columns of the matrix). Here we propose a lower
bound for the objective function of problem \ref{def:ercss} at each iteration, providing insight on the manner in which $\lambda$ interacts with the approximation error. This bound can be easily
extended to problem  
\ref{def:rcss}. 



\begin{lemma}  
Given a matrix $A \in \mathbb{R}^{m\times n}$, a regularization term
$\lambda \in \mathbb R$ and a set $S$ such that $|S|=k\leq n$. let $\sigma_i$ denote the $i$-th largest singular value of $A$. Then
\begin{equation}
\|A_{\bar S}-A_S(A_S^TA_S+\lambda I)^{-1}A_S^TA_{\bar S})\|_F^2 \geq \lambda^2\sum_{i=k+1}^n \left ( \frac{\sigma_i}{\sigma_i^2+\lambda} \right )^2
\end{equation}  
\end{lemma}
\begin{proof}  
Let us assume $A=A_S$, which corresponds to the best possible approximation we can obtain of the columns of $A$ using a subset of its columns as a basis. If $A=U\Sigma V^T$ is the
singular value decomposition of $A$,
$$
A-A(A^TA+\lambda I)^{-1}A^TA  = U\Sigma V^T-U\Sigma (\Sigma^2+\lambda I)^{-1} \Sigma^2V^T
$$
$$   = U\Sigma V^T-U
\left (\begin{array}{ccc}
\frac{\sigma_1^3}{\sigma_1^2+\lambda} &  &
\\ & \ddots &
\\  &  & \frac{\sigma_n^3}{\sigma_n^2+\lambda}
\end{array} \right ) V^T
$$

$$   = U \left (\begin{array}{ccc}
\frac{\sigma_1\lambda}{\sigma_1^2+\lambda} &  &
\\ & \ddots &
\\  &  & \frac{\sigma_n\lambda}{\sigma_n^2+\lambda}
\end{array} \right ) V^T
$$
The error incurred by approximating $A$ by itself using the regularized formulation is therefore exactly
$$
\|A-f_A([n], \lambda)\|_F^2 = \lambda^2\sum_{i=1}^n \left ( \frac{\sigma_i}{\sigma_i^2+\lambda} \right )^2
$$
Now consider that
\begin  {align*}
  \|A_{\bar S}-A_S(A_S^TA_S+\lambda I)^{-1}A_S^TA_{\bar S}\|_F^2 &\geq \|A_{\bar S}-A(A^TA+\lambda I)^{-1}A^TA_{\bar S}\|_F^2
  \\ & \geq \lambda^2\sum_{i=k+1}^n \left ( \frac{\sigma_i}{\sigma_i^2+\lambda} \right )^2
\end{align*}
The second inequality holds because of the interlacing inequalities of the singular values \cite{thompson1972principal}, and because $A_{\bar S}-A(A^TA+\lambda I)^{-1}A^TA_{\bar S}$ is an $m \times (n-k)$ submatrix of
$A-f_A([n], \lambda)$.
\end{proof}
To adapt this bound to problem \ref{def:rcss}, we simply need to extend the summation of the last inequality over all singular values.
The resulting expression for the error bound vanishes when $\lambda=0$ and approaches $\sum_i\sigma^2$ as $\lambda \rightarrow \infty$. This is of course consistent with the problem formulation. Observe that in the
first case, we are measuring the error incurred by approximating $A_S$ using its full span. In the second case, we approximate $A$ with a vanishing matrix, thus making the error equal to
$\|A\|_F^2=\sum_i\sigma_i^2$. 

This bound can be used to choose the value of $k$ if the input data set is suitable. If at some point the algorithm attains a value of the objective in problem \ref{def:ercss} that is close to this bound, then the present matrix has
almost as much representative power as the full column set, thus making the addition of more columns innecessary.

\section{Numerical experiments}
In order to validate our claims, we perform a series of numerical experiments. Specifically, we aim to assess the following aspects:
\begin{itemize}
\item \textit{Generalization ability}. We test the ability of the proposed algorithm to choose variables that can approximate well not only the input data (training data) but also future observations
  (test data). 
\item \textit {Stability}. We test how robust our algorithm is to noisy variations in the data.
\item \textit{Conditioning}. We measure the conditioning of the selected submatrix, defined as the ratio between the largest and the smallest singular values.
\item \textit{Running time}. We evaluate the running time of our algorithm with respect to different input parameters.
\item \textit{Applications}. We evaluate the effectiveness of our algorithm as a preprocessing step for clustering and its ability to reconstruct partially observed images.
\end{itemize}
To this end, we employed a variety of well-known data sets. We now briefly describe them, and indicate the preprocessing operations and the training/test splits for each of them.

\begin{itemize}
\item Isolet \cite{fanty1991spoken}. This data set consists of a collection of spoken letter recordings by various individuals, each represented by a set of features. The data were used as distributed. The variables are real-valued between -1 and 1. For the test set, we respected the split proposed by the authors.
\item MNIST \cite{lecun2010mnist}. Images of handwritten digits. The data were divided by 255 to ensure that all values be between 0 and 1. The training/test split provided by the authors was respected.
\item Yale Face Extended \cite{georghiades2001few}. Images of faces. The data were divided by 255 to ensure that all values be between 0 and 1. The first 1200 instances were used for training. The rest for testing.
\item ORL \cite{samaria1994parameterisation}. Ten different images of each of 40 distinct subjects. The data were divided by 255 to ensure that all values be between 0 and 1. The first 300 instances were used for training. The rest for testing.
\item COIL-20 \cite{nene1996columbia}. Images of objects of different
  categories. The data were divided by 255 to ensure that all values
  be between 0 and 1. The data were split into two halves for training
  and testing, ensuring class balance between both sets. We reduce the
  size of the images to 64$\times$64.
\item Online News Popularity \cite{fernandes2015proactive}. Statistics associated to news articles. All variables were standardized to zero mean and unit variance. The first 30,000 instances were used for training. The rest for testing.
\end{itemize}  
Table \ref{tab:datasets} summarizes the employed data sets.

\begin{table}[!hbt]
  \centering
  \caption{Employed datasets.}
\begin{tabular}{|c||c|c|c|}
\hline
Dataset & Variables & Train & Test \\
\hline
ORL & 1024 & 300 & 100
\\ \hline
MNIST & 784 & 60,000 & 10,000
\\ \hline
YaleB & 1024 & 1200 & 1214
\\ \hline
OnlineNews & 58 & 30,000 & 30,000
\\ \hline
IsoLet & 617 & 5200 & 1559
\\ \hline
COIL-20 & 4096 & 300 & 100
\\ \hline
\end{tabular}
\label{tab:datasets}
\end{table}

We consider two algorithms:
\begin{itemize}
\item GCSS: The unregularized greedy algorithm for problem
  \ref{def:cssp}. We use a Python implementation of the algorithm described by Farahat et al. \cite{farahat2011efficient}.
\item RGCSS: The algorithm proposed in this paper to solve problem
  \ref{def:ercss} (iterating until the desired number of columns is
  chosen). We use a Python implementation. In all experiments we use
  the algorithm adapted to optimize problem \ref{def:ercss}.
\end{itemize}

\subsection{Generalization ability}
\label{sec:exp_generalization}
In our first set of experiments, we evaluate whether the regularized formulation and the corresponding algorithm select columns that produce models with better generalization ability. To this end,
we run the algorithms, both GCSS and RGCSS, on small samples of the training splits of the data sets, and then measure the ability of the resulting models to approximate the rest of
the features of the test split.

Given an input matrix $A$, a test matrix $B$, a number $k \in \mathbb N$ and a value of $\lambda \in \mathbb R$, assume the algorithm being run outputs the set $S$. Then we measure the loss as
$$
L(S) = \|B-B_S(A_S^TA_S +\lambda I)^{-1}A_S^TA\|_F^2
$$
Notice that for both the unregularized and the regularized algorithms, we measure the loss using a regularized approximation. This is because even though the unregularized version does not attempt
to optimize this objective, the approximation on the test split will generally be much better if we introduce the regularization term $\lambda$, thus providing a fairer comparison. 
In doing so, we set the bar higher for our algorithm.

For all data sets, we run the algorithms on $2^i\%, i=0,\dots,4$ of the training data, for $k=2^i, i=4, \dots, 9$, and then measure the loss on the test split. It should be noted that we
simply set $\lambda=1$ for all cases. However, better results might be obtained by fine-tuning this parameter.
In order to assess the improvement brought by the regularized variant, we measure the relative improvement as follows. Let $S_N$ be the set output by the GCSS, and $S_R$ the
set output by RGCSS. 
Then the relative improvement is defined as
$$
100\times \frac{L(S_N)-L(S_R)}{L(S_N)}
$$
In order to provide a better estimate of this value, we run the algorithms 50 times on different random samples of the training set and replace $L(S_N)$ and $L(S_R)$ with their averages.

Figure \ref{fig:loss} illustrate the results. For each data set, we show the relative improvement for the different fractions of
the training set. It can be seen that the regularizing penalty yields significant improvements, especially when training data are scarce and $k$ is large. Notice the different scale on the plots
corresponding to OnlineNews and MNIST, where the improvement was more moderate.

\begin{figure}[hbt!]
    \subfloat[]{%
  \begin{minipage}{\linewidth}
  \includegraphics[width=.5\linewidth]{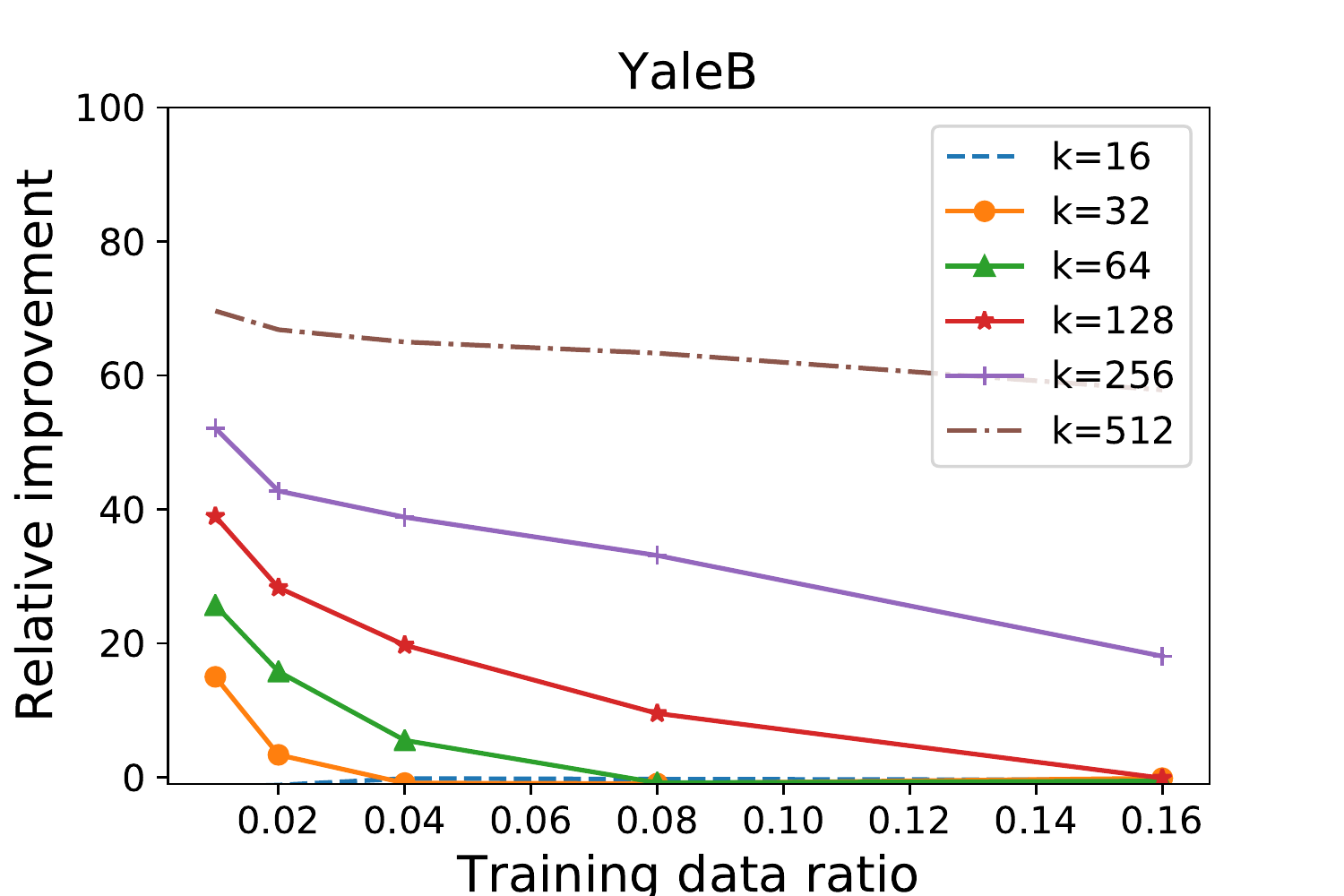}\hfill
  \includegraphics[width=.5\linewidth]{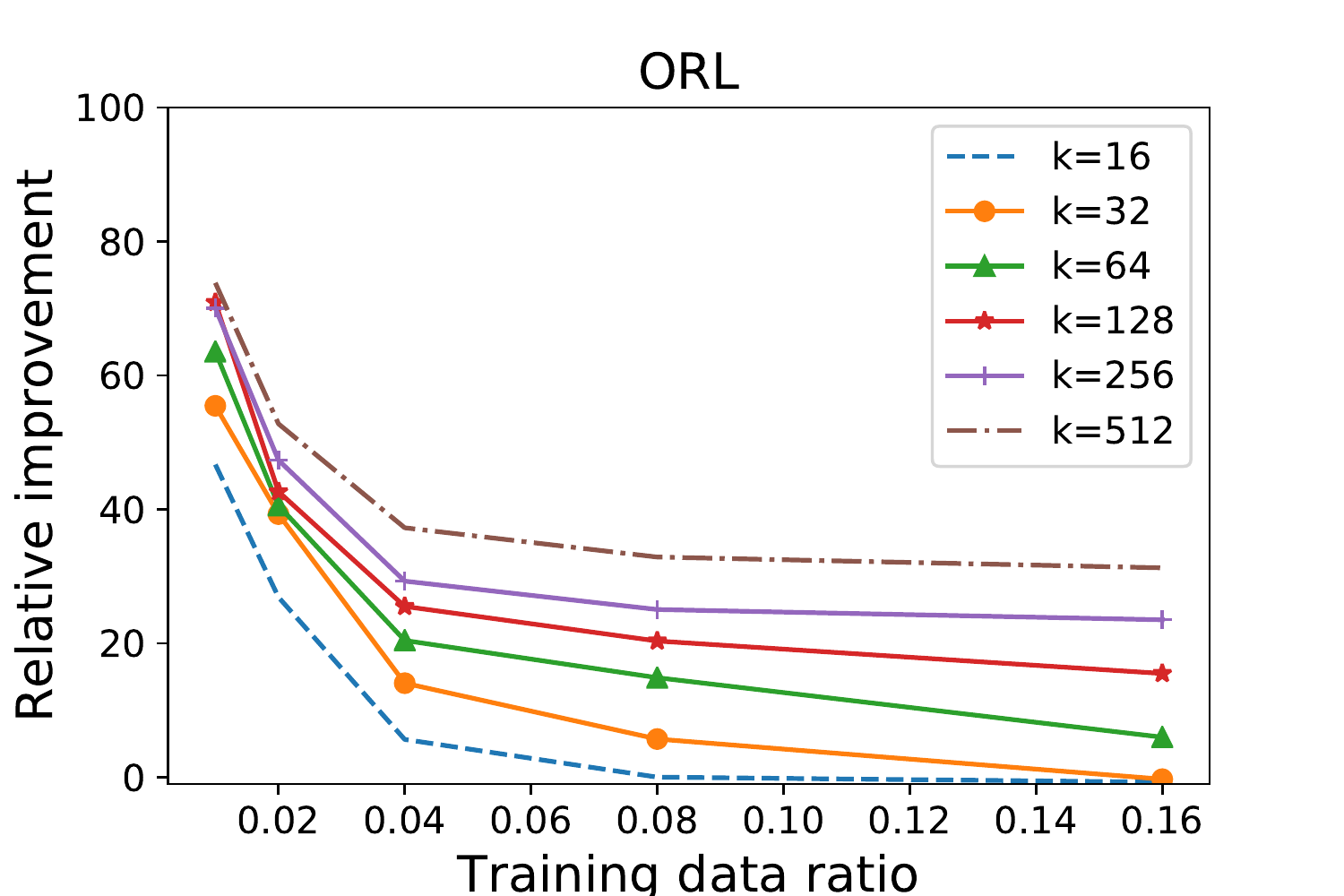}\hfill
  \end{minipage}}\par
    \subfloat[]{%
      \begin{minipage}{\linewidth}
  \includegraphics[width=.5\linewidth]{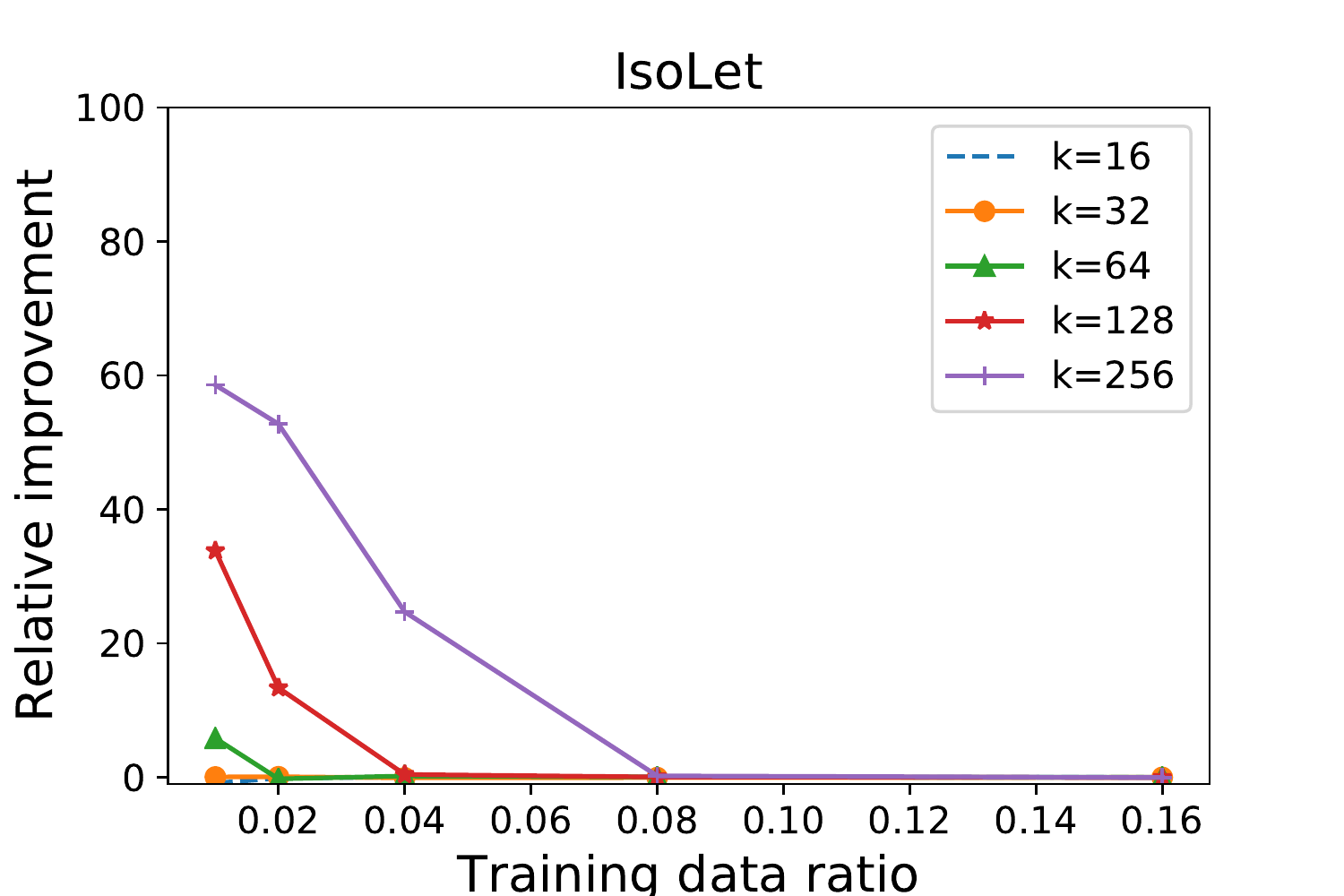}\hfill
  \includegraphics[width=.5\linewidth]{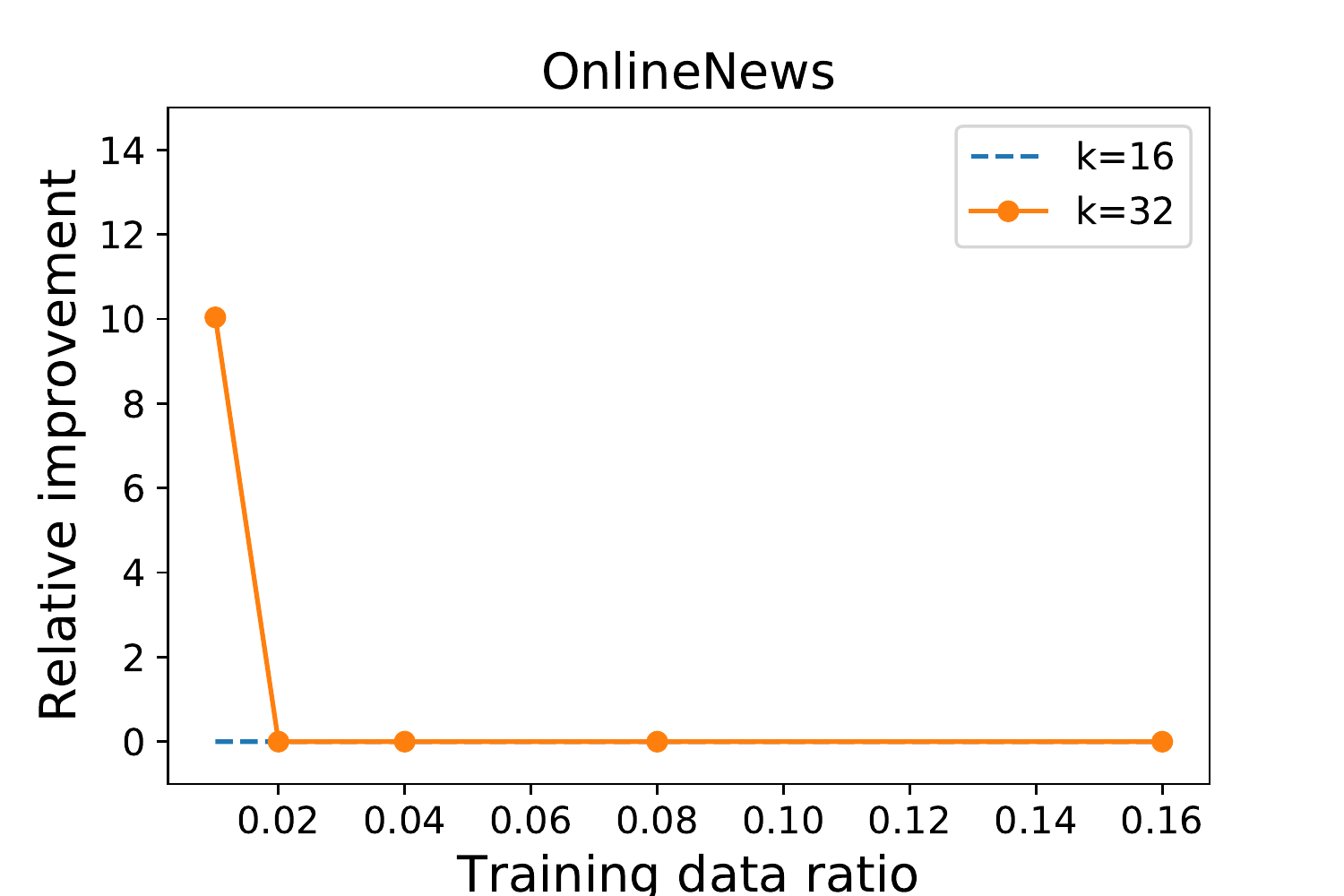}\hfill
    \end{minipage}}\par
    \subfloat[]{%
      \begin{minipage}{\linewidth}
        \centering
  \includegraphics[width=.5\linewidth]{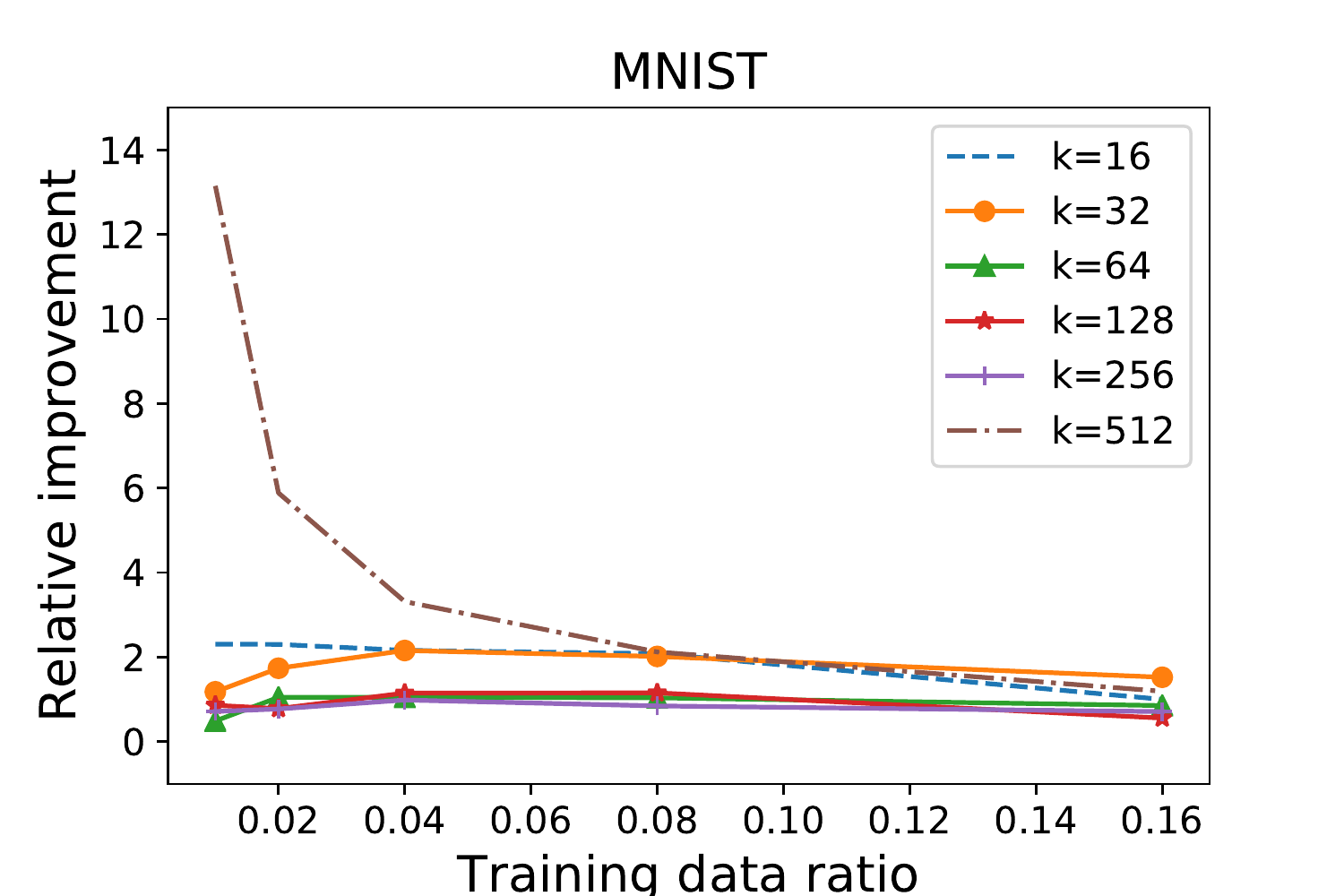}\hfill
  \end{minipage}}
    \caption{Relative improvement of the regularized variant}
    \label{fig:loss}
\end{figure}

\subsection{Stability}
\label{sec:exp_stability}
As discussed above, instances of the unregularized formulation of the column subset selection problem (problem \ref{def:cssp}) where $n \gg m$ are inconvenient.
In this set of experiments, we aim to verify whether the regularized formulation and the corresponding algorithm improves the stability of the results. In order to tests this, we perturb the
input data to see  how the algorithms behave in the face of noise.    

To measure the stability of each of the algorithms, we run them on $s$ different instances of the perturbed matrix and measure the average pairwise Jaccard index, which we define below.
Given two sets $S_1, S_2$, the Jaccard index is measured as
$$
J(S_1,S_2)=\frac{|S_1 \cap S_2|}{|S_1 \cup S_2|}
$$
Given a collection of sets $S_1, \dots, S_s$, we define the average pairwise Jaccard index as
$$
 \bar J(S_1, \dots, S_s)=\frac{1}{s(s-1)} \sum_i^s\sum_j^s J(S_i,S_j) \mathbb I\{i \neq j\}
$$
 where $I\{i \neq j\}$ is 1 if $i \neq j$, 0 otherwise. 
To make this index more meaningful, we calculate its expected value assuming the column subsets are chosen uniformly at random. Given a matrix of
 $n$ columns, assume we want to select a subset of size $k$. If we pick two subsets of $[n]$ at random, $S_1$ and $S_2$, there are ${n \choose k} ^2$ possible outcomes. Out of these, the number of pairs that have $k-p$ elements
 in common is
 $$
 {n \choose k} {k \choose p} {{n-k} \choose p}
 $$
 To see this, observe that for each of the ${n \choose k}$ possible values of $S_1$, $S_2$ must have $p$ out of $k$ elements that are not in $S_1$, and those can be any of the
 remaining $n-k$ ones.
 Therefore, the expected value of the size of the intersection between two subsets drawn uniformly at random is
 $$
\mathbb E[S_1\cap S_2] = \sum_{p=0}^k \frac{ (k-p) {k \choose p} { {n-k} \choose p} }{{n \choose k}}
 $$
Now, the Jaccard index of each of those pairs is the size of the intersection divided by the size of the union. Hence, given $n$ and $k$,
$$
\mathbb E[J]=\sum_{p=0}^k \frac{ (k-p) {k \choose p} { {n-k} \choose p} }{{n \choose k}(k+p)}
$$
We can now measure the stability of the algorithms by running them on different perturbations of the input matrix, and then comparing the average pairwise Jaccard index of the resulting subsets with
the expected value of the Jaccard index.
We consider the case where $n>m$, that is, the input matrix has more columns than rows. To this end, we take random samples of 100 rows of each training data set and set $k=m$. In the case of OnlineNews,
since $n<100$, we take $k=m=n/2$. Note that the case $k>m$ becomes pathological in the unregularized formulation, since any column choice once the span of the data has been covered is equally
inocuous. Therefore, as $k$ grows beyond the value of $m$, the Jaccard index for the unregularized formulation will approach $\mathbb E[J]$ if ties are broken arbitrarily.

We take the input data set and perturb it with a matrix whose entries are independently sampled from a Gaussian distribution with zero mean and a standard deviation of $10^{-3}$. As explained above, we
apply $s$ different perturbations to the input data and run the algorithms on each of them, thus obtaining $s$ different subsets for each algorithm. We set $s=100$ and measure the average pairwise Jaccard index.
Table \ref{tab:stability} shows the results. We also show the expected value of the Jaccard index to know how close to a random choice each algorithm is. We run this experiment for
$\lambda=0$ (i.e. the unregularized algorithm by Farahat et al. \cite{farahat2011efficient}), $\lambda=1$ and $\lambda=10$. It can be seen that the regularized formulation significantly improves the stability of the results.

\subsection{Conditioning}
The conditioning of a matrix can be loosely understood as a measure of numerical rank defficiency. Formally, given a matrix $C$ of rank $k$, 
we define its conditioning, or its condition number, as
$\kappa(C)=\frac{\sigma_1(C)}{\sigma_k(C)}$. Ill-conditioned matrices,
that is, with a large condition number, are prone to significant numerical errors when involved in the solution of linear systems.

We compute the condition number of the submatrices of the training set selected by both
algorithms, regularized and unregularized. We run the algorithm 50 times on
different random samplings of the training set and report the minimum,
average and maximum across all runs in table \ref{tab:conditioning}.
The results clearly reveal that the regularization term encourages
the selection of significantly better-conditioned column subsets.

In order to avoid overcrowding the table we only report the results
for $k=16,32$, as they are illustrative of the general behaviour of
the algorithms in this regard.

An interesting fact revealed by our experiments is the following: in
cases where $k>m$, the unregularized formulation of the problem is
ill-posed. In terms of the objective function, once the
$m$-dimensional subspace spanned by the matrix has been covered, any
subsequent column choice is equally good. The unregularized algorithm
therefore yields particularly poorly conditioned subsets (see e.g. YaleB,
$0.01*m$, $k=16$, $\lambda=0$ in table \ref{tab:conditioning}. What is
surprising is that in these situations, the regularized variant
produces column subsets that lead to well-conditioned matrices
\textit{even in the test set}. An example of the obtained condition
numbers on submatrices of the test set is shown in table
\ref{tab:test_cond}.

\begin{table}[!hbt]
\footnotesize
  \centering
  \caption{Condition number of the matrices output by the unregularized and the regularized algorithms. For each experiment we report the minimum, average and maximum (min / avg / max) of 50 runs on different random samplings.}
\begin{tabular}{|c|c|| c|c| c|c|}
\hline
& & \multicolumn{2}{|c|}{$k=16$} & \multicolumn{2}{|c|}{$k=32$}\\
\hline
Sample size & Dataset & $GCSS$ & $RGCSS$ & $GCSS$ & $RGCSS$  \\
\hline 
\multirow{ 5}{*}{0.01*m} 
& ORL & 13.07 / 338.94 / 3127.57 & 4.45 / 6.28 / 10.71 & 33.92 / 116.77 / 240.82 & 3.90 / 5.10 / 7.68
\\ \cline{2-6}
& MNIST & 130.88 / 142.34 / 153.50 & 5.74 / 6.13 / 6.56 & 10.22 / 146.46 / 241.65 & 10.22 / 10.90 / 11.80
\\ \cline{2-6}
& YaleB & 243.81 / 12186.18 / 48241.03 & 31.97 / 183.84 / 1019.68 & 214.28 / 1218.04 / 2575.45 & 10.10 / 63.30 / 272.59
\\ \cline{2-6}
& OnlineNews & 2.37 / 3.63 / 5.33 & 2.37 / 3.39 / 4.59 & 4.67 / 5.86 / 7.43 & 4.58 / 5.71 / 7.32
\\ \cline{2-6}
& Isolet & 12.66 / 19.91 / 29.78 & 11.58 / 14.75 / 17.96 & 43.04 / 77.20 / 133.39 & 31.35 / 39.31 / 58.30
\\ \hline \hline
\multirow{ 5}{*}{0.04*m} 
& ORL & 119.35 / 2331.63 / 20602.93 & 55.78 / 90.22 / 187.40 & 183.43 / 985.98 / 4140.42 & 20.59 / 27.88 / 42.15
\\ \cline{2-6}
& MNIST & 136.33 / 147.26 / 161.20 & 5.61 / 5.90 / 6.27 & 227.12 / 235.81 / 247.82 & 9.36 / 10.03 / 10.40
\\ \cline{2-6}
& YaleB & 29.62 / 35.36 / 46.36 & 24.32 / 32.24 / 41.23 & 99.93 / 134.32 / 185.35 & 76.28 / 103.24 / 139.18
\\ \cline{2-6}
& OnlineNews & 2.46 / 3.17 / 4.07 & 2.46 / 3.01 / 3.91 & 3.99 / 4.82 / 5.86  & 3.98 / 4.72 / 5.51
\\ \cline{2-6}
& Isolet & 10.82 / 13.24 / 16.66 & 10.42 / 12.52 / 16.53 & 20.40 / 24.28 / 28.62  & 20.38 / 23.58 / 27.96
\\ \hline \hline
\multirow{ 5}{*}{0.16*m} 
& ORL & 61.95 / 74.08 / 90.82 & 41.89 / 60.14 / 86.10 & 178.82 / 233.01 / 304.82 & 131.69 / 175.00 / 210.30
\\ \cline{2-6}
& MNIST & 130.80 / 146.03 / 151.36 & 5.55 / 18.69 / 134.06 & 225.88 / 235.61 / 244.84 & 9.95 / 55.10 / 241.62
\\ \cline{2-6}
& YaleB & 22.11 / 24.50 / 28.54 & 20.93 / 24.29 / 28.38 & 40.26 / 43.60 / 51.41 & 38.01 / 41.91 / 48.99
\\ \cline{2-6}
& OnlineNews & 2.45 / 2.80 / 3.34 & 2.45 / 2.80 / 3.34 & 3.78 / 4.38 / 5.12  & 3.77 / 4.24 / 5.12
\\ \cline{2-6}
& Isolet & 9.58 / 10.74 / 11.93 & 9.58 / 10.74 / 11.93 & 17.82 / 20.30 / 21.42  & 17.82 / 20.30 / 21.42
\\ \hline
\end{tabular}
\label{tab:conditioning}
\end{table}

\begin{table}[!hbt]
\footnotesize
  \centering
  \caption{Condition number of the test submatrices chosen by the
    unregularized and the regularized algorithms, in a case where
    $k>m$. In this situation, the unregularized version of the problem
    becomes ill-posed, and the algorithm unstable.}
\begin{tabular}{|c|c|| c|c| c|c|}
\hline
& & \multicolumn{2}{|c|}{$k=16$} & \multicolumn{2}{|c|}{$k=32$}\\
\hline
Sample size & Dataset & $\lambda=0$ & $\lambda=1$ & $\lambda=0$ & $\lambda=1$  \\
\hline 
0.01*m & YaleB & 6.22e+16 / 8.76e+16 / 2.37e+17 & 17.43 / 26.07 / 45.11 & 1.21e+16 / 1.74e+16 / 2.84e+16 & 28.34 / 46.74 / 72.87
\\ \hline 
\end{tabular}
\label{tab:test_cond}
\end{table}

\subsection{Clustering}
We test the effectiveness of our algorithm as a preprocessing step for clustering. Dimensionality reduction is often essential for these tasks,
because the distance computations employed by most clustering algorithms are particularly sensitive to large numbers of variables. 

In order to evaluate the ability of our methods to produce robust feature subsets, we proceed as follows: we run the algorithm on a small
portion of the training set (of varying size) and then reduce the test set so as to keep the chosen variables only. We then run the $k$-means clustering 
algorithm on this reduced data set. For reference, we also consider
the case where $k=n$, that is, using the whole feature set for
clustering. Note that in this case, the training split does not play a
part in the result. The results for the different test set sizes,
equal to $m-$(training set size), are thus expected to be similar.

We considered the data sets ORL, COIL20 and IsoLet. We discarded MNIST and YaleB, where the $k$-means algorithm did not produce acceptable
results, and OnlineNews, whose target values are better suited to a regression task.

To measure the quality of the result, we compute the normalized mutual information (NMI) of the ground truth labels and the obtained 
partition. The NMI is defined as follows. Given two discrete random variables $X,Y$, the NMI is defined as follows:
$$
NMI(X,Y) = \frac{MI(X,Y)}{\sqrt{H(X)H(Y)}}
$$
where $MI(X,Y)$ is the mutual information of $X$ and $Y$, and $H(X)$ is the entropy of $X$.

Cluster centroids were initialized using the $k$-means++ scheme, and the best result out of 10 in terms of
the objective function was kept. The process was repeated 50 times, running the column subset selection algorithms on different random
samples of the training set each time. We report the average of the obtained NMI values.

The results are shown in table \ref{tab:clust}, using from 1\% to 16\% of the training data. In all 3 data sets, RGCSS shows superior 
performance. Some of the results warrant further discussion. First, an interesting property of RGCSS is that the NMI is remarkably
stable with respect to the amount of training data used, while GCSS
generally only starts obtaining good results when a sizeable portion is 
employed. Second, in a few instances, RGCSS did show clearly poorer performance (COIL20, $k=16,32$). It would be interesting to determine
the cause of this defficiency.

This results provide evidence of the clear advantages of using the
regularized variant of column subset selection for practical
applications. In particular, note how the quality of the clustering
improves when using feature subsets of size $128$ or more rather than
the entire feature set. In the case
of RGCSS, this improvement is present even when the feature subset
was chosen using only 1\% of the training data.

\begin{table}[!hbt]
\footnotesize
  \centering
  \caption{NMI for clustering results on the test set using the feature subset chosen by each algorithm.}
\begin{tabular}{|c|| c|c| c|c| c|c| c|c| c|c|}
\hline
\multicolumn{11}{|c|}{ORL}\\
\hline
& \multicolumn{2}{|c|}{$0.01*m$} & \multicolumn{2}{|c|}{$0.02*m$} & \multicolumn{2}{|c|}{$0.04*m$} & \multicolumn{2}{|c|}{$0.08*m$} & \multicolumn{2}{|c|}{$0.16*m$}
\\ \hline
k & GCSS & RGCSS & GCSS & RGCSS & GCSS & RGCSS & GCSS & RGCSS & GCSS & RGCSS 
\\ \hline \hline
16 & 0.67 & 0.768 & 0.733 & 0.777 & 0.786 & 0.787 & 0.78 & 0.785 & 0.797 & 0.796
\\ \hline
32 & 0.653 & 0.791 & 0.713 & 0.798 & 0.76 & 0.806 & 0.809 & 0.809 & 0.817 & 0.805
\\ \hline
64 & 0.632 & 0.808 & 0.685 & 0.808 & 0.734 & 0.816 & 0.786 & 0.817 & 0.821 & 0.817
\\ \hline
128 & 0.614 & 0.819 & 0.657 & 0.817 & 0.7 & 0.815 & 0.754 & 0.818 & 0.8 & 0.819
\\ \hline
256 & 0.588 & 0.821 & 0.625 & 0.82 & 0.672 & 0.82 & 0.721 & 0.819 & 0.767 & 0.821
\\ \hline
512 & 0.565 & 0.824 & 0.596 & 0.823 & 0.639 & 0.825 & 0.69 & 0.821 & 0.739 & 0.821
\\ \hline
n & \multicolumn{2}{c|}{0.81} & \multicolumn{2}{c|}{0.808} & \multicolumn{2}{c|}{0.803} & \multicolumn{2}{c|}{0.811} & \multicolumn{2}{c|}{0.809}
\\ \hline
\end{tabular}

\begin{tabular}{|c|| c|c| c|c| c|c| c|c| c|c|}
\hline
\multicolumn{11}{|c|}{COIL20}\\
\hline
& \multicolumn{2}{|c|}{$0.01*m$} & \multicolumn{2}{|c|}{$0.02*m$} & \multicolumn{2}{|c|}{$0.04*m$} & \multicolumn{2}{|c|}{$0.08*m$} & \multicolumn{2}{|c|}{$0.16*m$}
\\ \hline
k & GCSS & RGCSS & GCSS & RGCSS & GCSS & RGCSS & GCSS & RGCSS & GCSS & RGCSS 
\\ \hline \hline
16 & 0.644 & 0.636 & 0.696 & 0.636 & 0.718 & 0.623 & 0.733 & 0.629 & 0.739 & 0.604
\\ \hline
32 & 0.642 & 0.686 & 0.703 & 0.67 & 0.744 & 0.684 & 0.762 & 0.676 & 0.769 & 0.672
\\ \hline
64 & 0.654 & 0.721 & 0.701 & 0.724 & 0.739 & 0.726 & 0.774 & 0.727 & 0.783 & 0.723
\\ \hline
128 & 0.639 & 0.76 & 0.701 & 0.769 & 0.744 & 0.78 & 0.774 & 0.772 & 0.789 & 0.769
\\ \hline
256 & 0.646 & 0.782 & 0.707 & 0.782 & 0.74 & 0.787 & 0.772 & 0.794 & 0.79 & 0.795
\\ \hline
512 & 0.662 & 0.782 & 0.698 & 0.779 & 0.746 & 0.777 & 0.77 & 0.779 & 0.793 & 0.783
\\ \hline
n & \multicolumn{2}{c|}{0.758} & \multicolumn{2}{c|}{0.755} & \multicolumn{2}{c|}{0.755} & \multicolumn{2}{c|}{0.754} & \multicolumn{2}{c|}{0.756}
\\ \hline
\end{tabular}

\begin{tabular}{|c|| c|c| c|c| c|c| c|c| c|c|}
\hline
\multicolumn{11}{|c|}{IsoLet}\\
\hline
& \multicolumn{2}{|c|}{$0.01*m$} & \multicolumn{2}{|c|}{$0.02*m$} & \multicolumn{2}{|c|}{$0.04*m$} & \multicolumn{2}{|c|}{$0.08*m$} & \multicolumn{2}{|c|}{$0.16*m$}
\\ \hline
k & GCSS & RGCSS & GCSS & RGCSS & GCSS & RGCSS & GCSS & RGCSS & GCSS & RGCSS 
\\ \hline \hline
16 & 0.573 & 0.571 & 0.576 & 0.575 & 0.579 & 0.577 & 0.583 & 0.583 & 0.589 & 0.589
\\ \hline
32 & 0.604 & 0.602 & 0.608 & 0.61 & 0.604 & 0.608 & 0.601 & 0.603 & 0.596 & 0.597
\\ \hline
64 & 0.642 & 0.655 & 0.66 & 0.665 & 0.656 & 0.662 & 0.666 & 0.665 & 0.662 & 0.663
\\ \hline
128 & 0.577 & 0.703 & 0.698 & 0.709 & 0.707 & 0.715 & 0.712 & 0.71 & 0.711 & 0.712
\\ \hline
256 & 0.5 & 0.725 & 0.613 & 0.733 & 0.728 & 0.735 & 0.734 & 0.733 & 0.733 & 0.736
\\ \hline
n & \multicolumn{2}{c|}{0.697} & \multicolumn{2}{c|}{0.699} & \multicolumn{2}{c|}{0.704} & \multicolumn{2}{c|}{0.704} & \multicolumn{2}{c|}{0.701}
\\ \hline
\end{tabular}
\label{tab:clust}
\end{table}

\subsection{Image reconstruction}

In order to provide a visual account of the improved generalization ability of the proposed method, we evaluate its ability to reconstruct
instances of unseen images. To this end, we run GCSS and RGCSS on a small portion of a 64$\times$64 version of the ORL data set. 

We proceed as follows. Let $A$ be the portion of the training data set used, and let $S$ be the subset output by the employed algorithm.
We define $C=A_S$ and  $W=(C^TC+\lambda I)^{-1}C^TA$. To rebuild an instance of the test set $x$, we compute 
$$
x_S^TW
$$

Three examples are shown in figure \ref{fig:recons}. The reconstructions
were done for $k=128,256,512,1024,2048$. The shown images were selected as
follows. We compute the average reconstruction error attained by each
algorithm at each value of $k$. The first image we show is the one
whose reconstruction error was the closest to the average for $k=2048$
(which was casually the same for both algorithms). The second image
was the closest to the average error using $GCSS$ for $k=128$. The
third, the closest to the average error using $RGCSS$ for $k=128$. The
chosen images should therefore be close to the expected reconstruction
error incurred by each algorithm for $k=128$ and $k=2048$.

These images demonstrate how the models obtained using RGCSS exhibit a better ability
to recover certain visual characteristics. In particular, this highlights a
previously discussed issue (see section \ref{sec:exp_stability}). When $k$ becomes larger than the rank of
the input matrix, the unregularized formulation of the problem no
longer constitutes a suitable objective to judiciously choose
additional columns. This is visible in this example when $k>512$.
Increasing the value of $k$ does not provide improvements when using
GCSS, while RGCSS manages to recover additional nuance when
the dimensionality of the model increases (even if the result is
visually subtle, the reconstruction error continues to decrease using
RGCSS, while it ceases to do so with GCSS).

\begin{figure}[hbt!]
    \subfloat[]{%
  \begin{minipage}{\linewidth}
  \includegraphics[width=\linewidth]{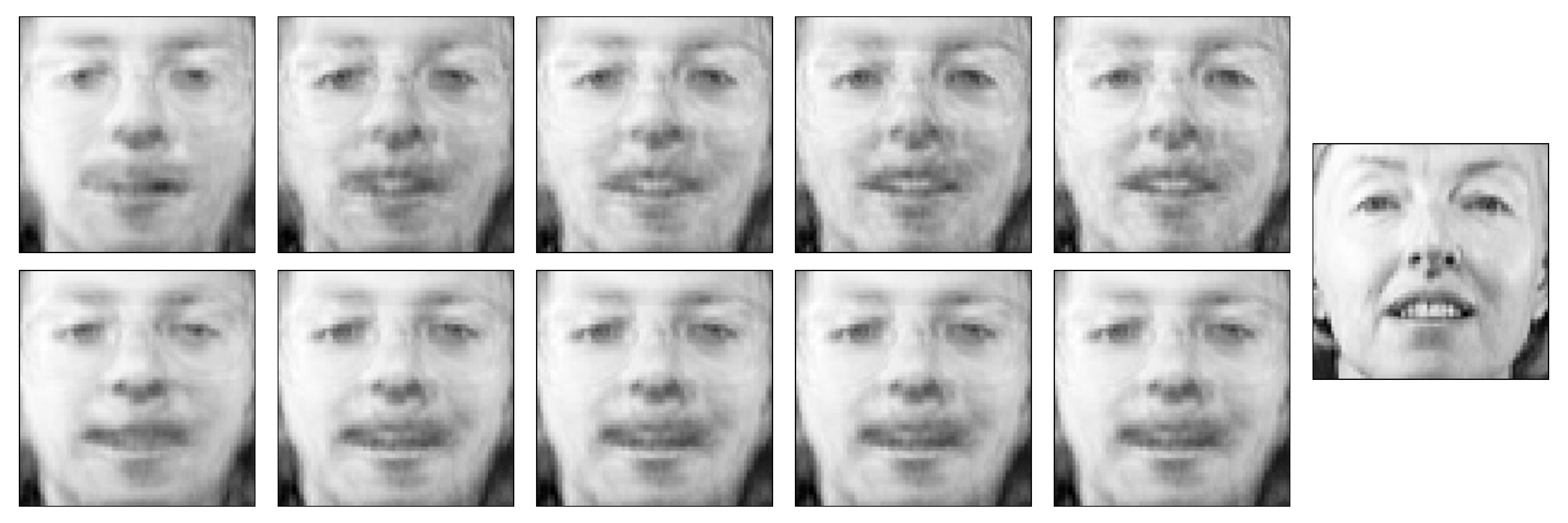}\hfill
  \end{minipage}}\par
    \subfloat[]{%
  \begin{minipage}{\linewidth}
  \includegraphics[width=\linewidth]{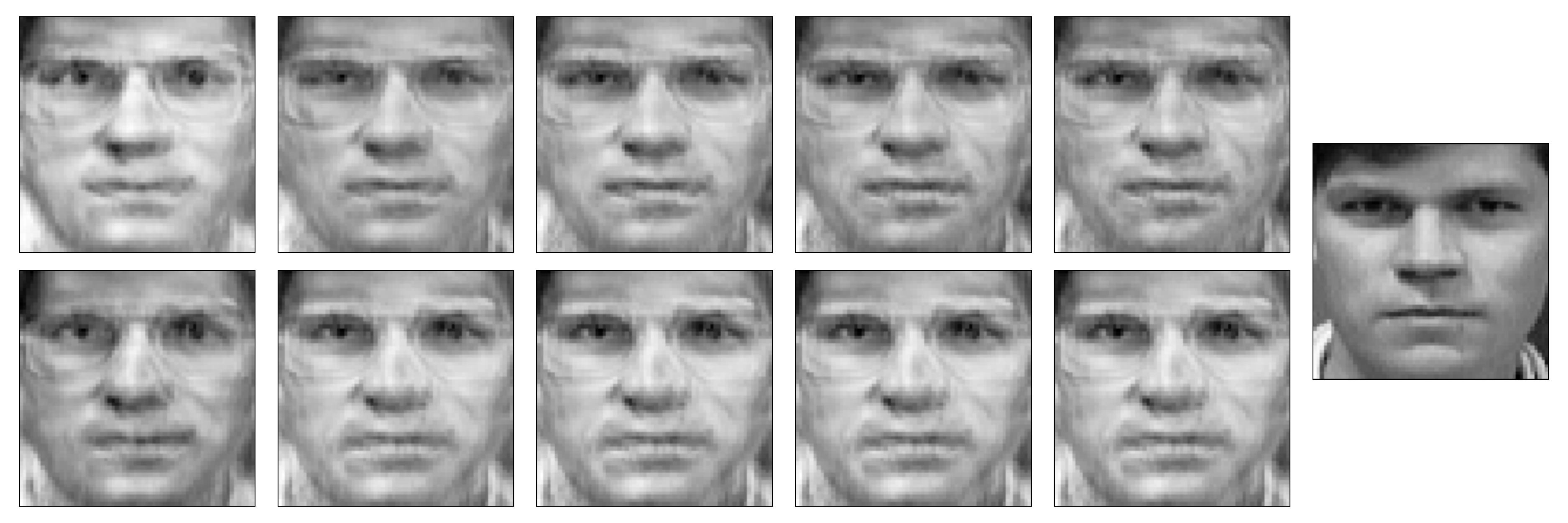}\hfill
  \end{minipage}}\par
    \subfloat[]{%
      \begin{minipage}{\linewidth}
  \includegraphics[width=\linewidth]{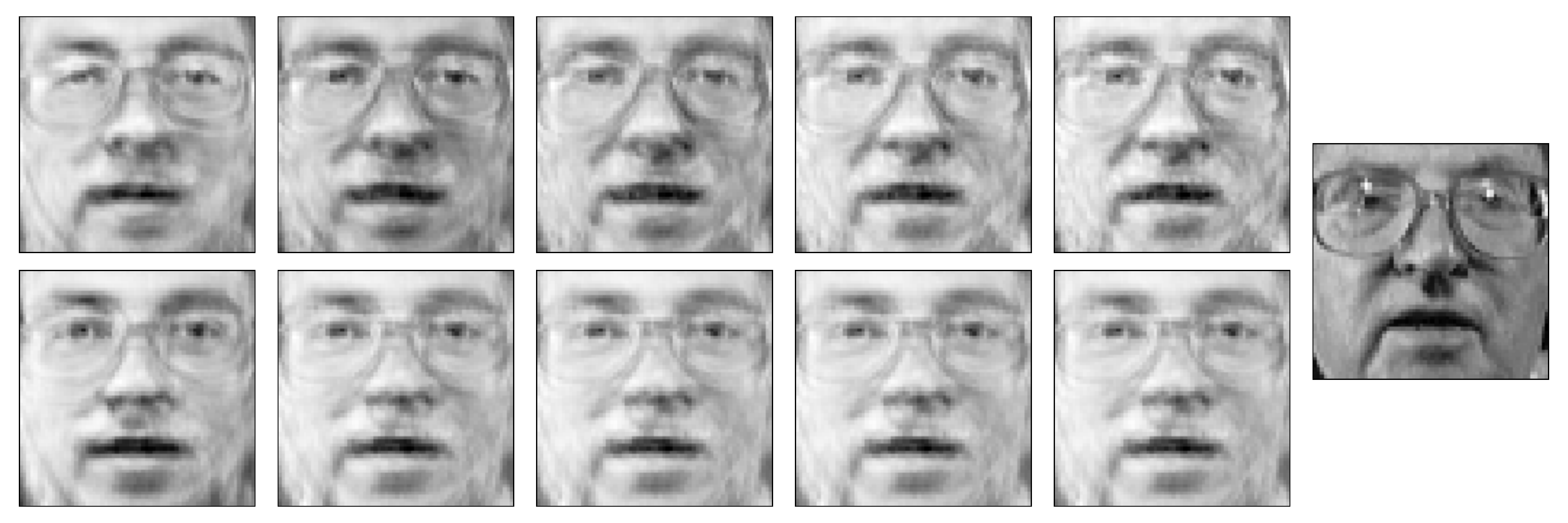}\hfill
    \end{minipage}}\par
\caption{Reconstruction of test set images. For both (a), (b) and (c): Top: reconstruction by RGCSS. Bottom: reconstruction by GCSS. Right: Original image. We show reconstructions for $k=128,256,512,1024,2048$.}
\label{fig:recons}
\end{figure}

\subsection{Running time}
In order to evaluate the running time of our algorithm, we generated synthetic matrices of increasing dimensions and ran the algorithms for increasing values of $k$. Figure \ref{fig:time} shows the
resulting times for the unregularized (GCSS) and the regularized (RGCSS) algorithms.
 To measure the sensitivity with respect to each parameter, we fixed the other two. Specifically, we run the following
experiments:
\begin{itemize}
\item $n=1000,k=128$, $m \in [10^3,10^5]$
\item $m=1000,k=1$, $n \in [100,5000]$
\item $m=1000,n=1024$, $k \in [10,1020]$
\end{itemize}    
The behavior of both algorithms with respect to $m$ and $n$ (the number of rows and columns respectively) is very similar. For large values of $k$, the regularized variant does exhibit noticeably larger
running times. It should be noted, however, that the ratio between the
time required by the two algorithms converges to a constant factor.
This ratio is shown in the plot for varying values of $k$ to support
this claim.

\begin{figure}
    \subfloat[]{%
  \begin{minipage}{\linewidth}
  \includegraphics[width=.5\linewidth]{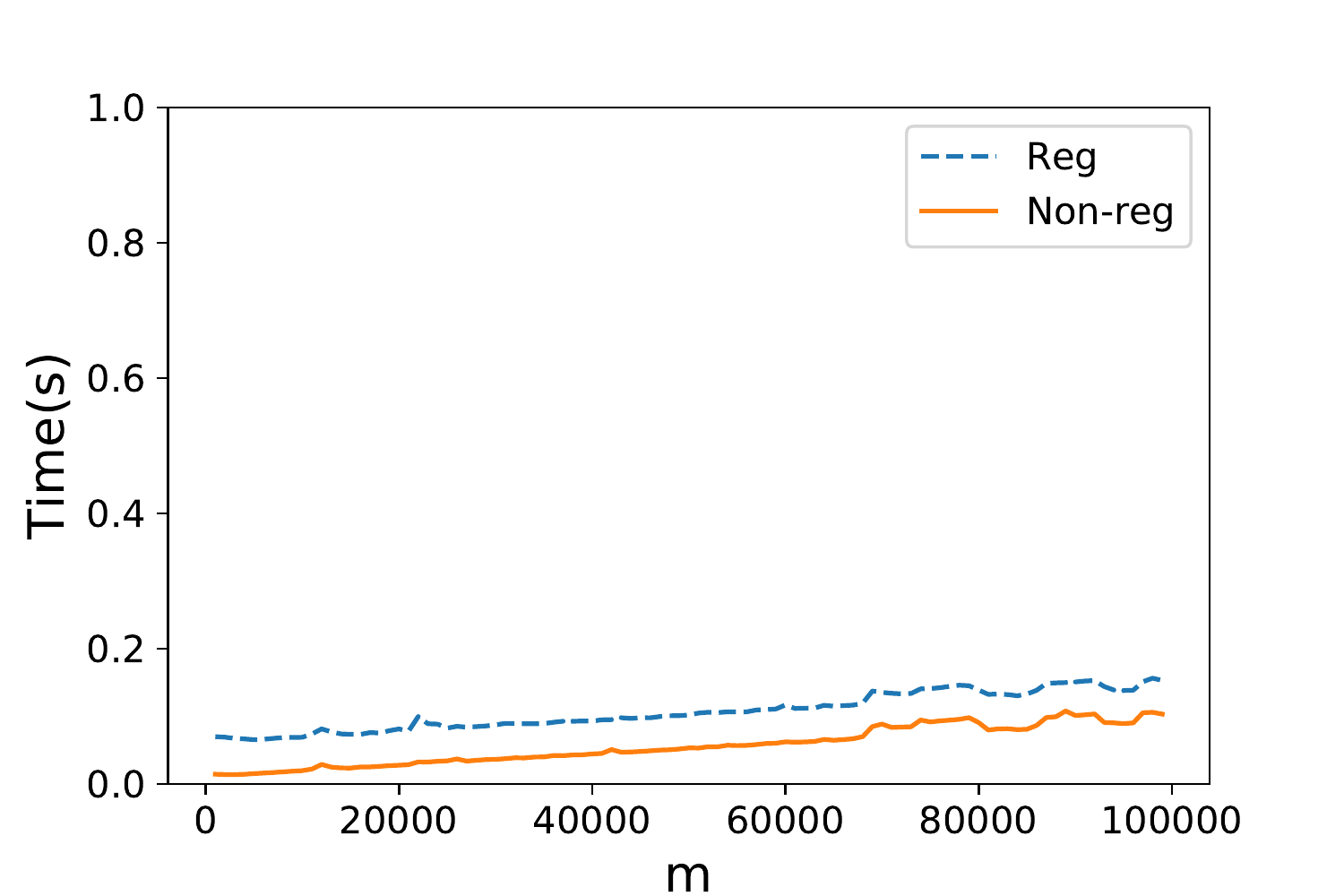}\hfill
  \includegraphics[width=.5\linewidth]{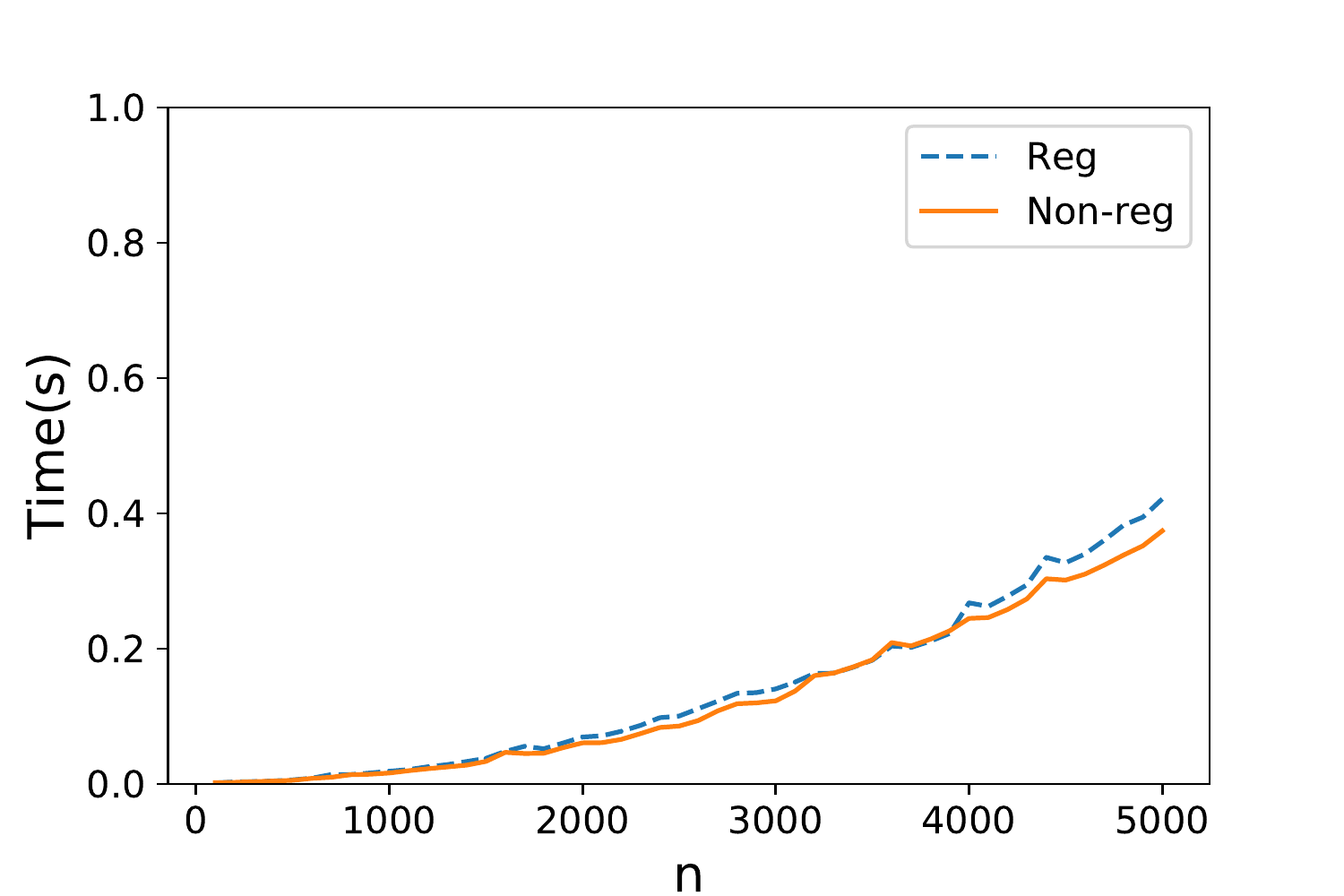}\hfill
  \end{minipage}}\par
    \subfloat[]{%
  \begin{minipage}{\linewidth}
    \centering
    \includegraphics[width=.5\linewidth]{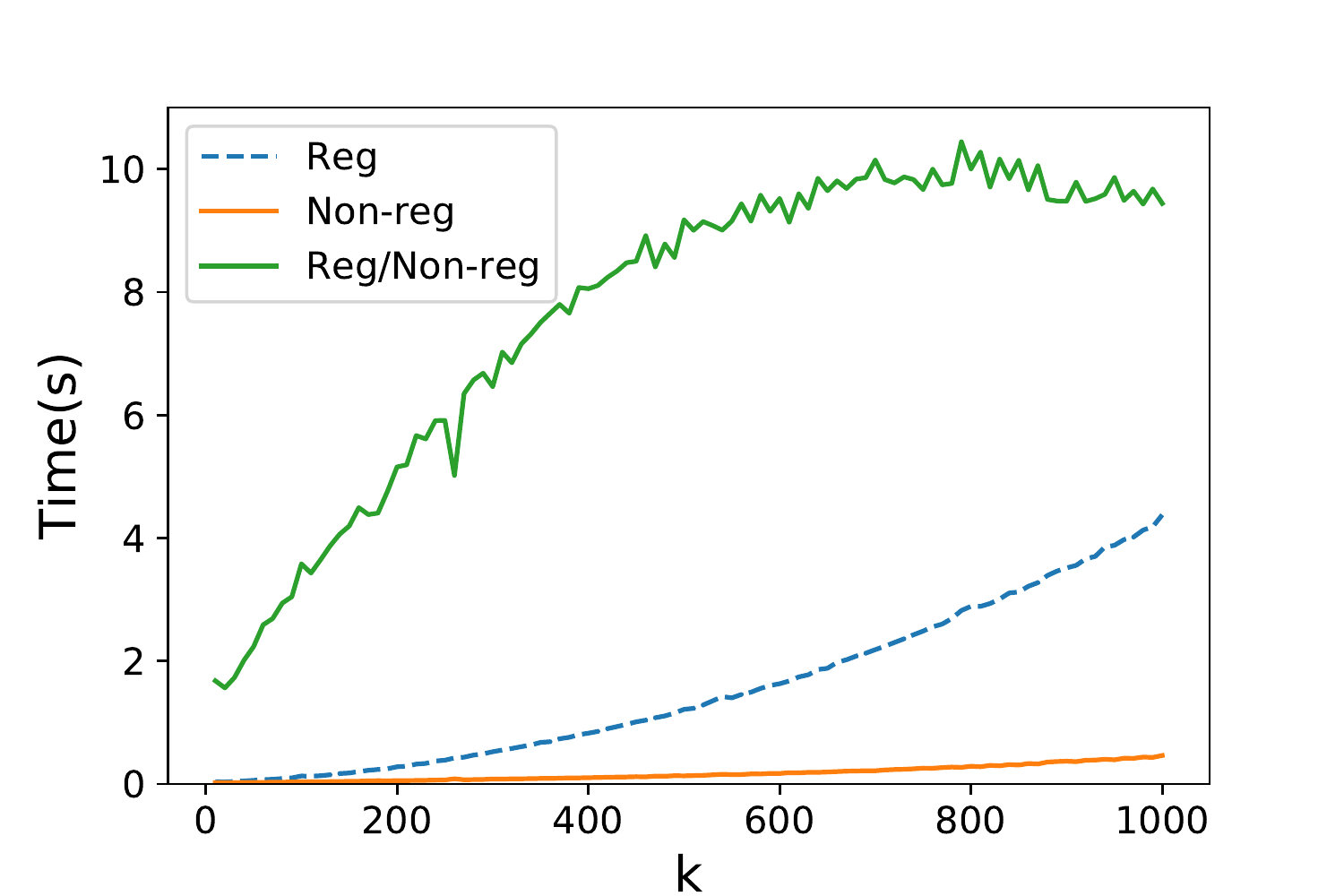}\hfill
  \end{minipage}}\par
    \caption{Running times with respect to the different input
      parameters. The ratio between the running time of the two
      algorithms is shown in the plot for the value of $k$.}
    \label{fig:time}
\end{figure}

\begin{table}[!hbt]
  \centering
  \caption{Average pairwise Jaccard index and loss.}
\begin{tabular}{|c|| c|c|c| c|}
\hline
& \multicolumn{3}{|c|}{$\bar J$} &  \\
\hline
Dataset & $\lambda=0$ & $\lambda=1$ & $\lambda=10$ & $\mathbb E[J]$ \\
\hline
ORL & 0.177 & 0.245 & 0.408 & 0.012
\\ \hline
MNIST & 0.278 & 0.951 & 0.774 & 0.068
\\ \hline
YaleB & 0.252 & 0.632 & 0.852 & 0.051
\\ \hline
OnlineNews & 0.782 & 1.0 & 1.0  & 0.336
\\ \hline
Isolet & 0.422 & 0.6 & 0.838  & 0.088
\\ \hline
\end{tabular}
\label{tab:stability}
\end{table}

\section{Conclusions \& future work}
In this paper we have presented a novel formulation of the Column
Subset Selection Problem that incorporates a regularization term. We
have derived an efficient algorithm to greedily optimize it,
and have demonstrated its potential through various experiments. In
addition, we have discussed how this formulation can be inadequate for feature selection and have proposed an alternative that solves
this problem. Finally, we have derived a lower bound for the error of the proposed problem. We believe that these new problem formulations open exploration directions with regards to column subset
selection. The advantages of using the proposed algorithm in practice
have been demostrated by a wide variety of experiments. In future work, it would be interesting to study the impact of the value of $\lambda$ on the generalization ability of the resulting models, and whether significant improvement can be
gained by fine-tuning. Additionally, it would be interesting to know whether optimal values can be derived making distributional assumptions with respect to the input data. Finally, it would be
interesting to study the possibility of deriving approximation guarantees for the greedy and other algorithms.



\appendix
\section*{Appendix A.}
Here we show more detailed derivations of some of the equalities in the paper.

Equality \ref{eq:update_At}:
\begin{align*}
  & A^{(t+1)} =   C_w( C_w^T  C_w + \lambda I)^{-1}   C_w^TA 
  \\ = ~&  (C\mbox{ } w)(C_w^T C_w + \lambda I)^{-1}   (C\mbox{ } w)^TA 
  \\ = ~& C(C^TC+\lambda I)^{-1}C^TA + C\frac{vv^T}{\alpha_w}C^TA  -  w\frac{v^T}{\alpha_w}C^TA - C\frac{vw^T}{\alpha_w}A + \frac{ww^T}{\alpha_w}A 
  \\ = ~& A^{(t)} + \frac{(A_{:w}A^T_{:w})^{(t)}}{\alpha_w}A  -  w\frac{(A^T_{:w})^{(t)}}{\alpha_w}A - \frac{A^{(t)}_{:w}w^T}{\alpha_w}A + \frac{ww^T}{\alpha_w}A 
  \\ = ~& A^{(t)} + \frac{1}{\alpha_w}\left ( (A^{(t)}_{:w} - w  )(A^{(t)}_{:w} - w )^T \right )A
\end{align*}
Equality \ref{eq:preargmin}:
\begin{align*}
&   \argmin \|A-A^{(t+1)}\|_F^2 
\\=& \argmin_i tr \left ((A-A^{(t+1)})^T(A-A^{(t+1)}) \right ) 
\\ =& \argmin_i tr (A^TA) - tr(A^TA^{(t+1)}) - tr((A^T)^{(t+1)}A^T) + tr((A^TA)^{(t+1)})
\\ =& \argmin_i tr (A^TA) - tr(A^T(A^{(t)}+\frac{1}{\alpha_i}d_id_i^TA)) 
\\&- tr((A^{(t)}+\frac{1}{\alpha_i}d_id_i^TA)^TA) + tr((A^{(t)}+\frac{1}{\alpha_i}d_id_i^TA)^T(A^{(t)}+\frac{1}{\alpha_i}d_id_i^TA))
\\ =& \argmin_i tr (A^TA) - tr(A^TA^{(t)}) - tr(\frac{1}{\alpha_i}A^Td_id_i^TA)) 
\\&- tr((A^T)^{(t)}A) - tr(\frac{1}{\alpha_i}A^Td_id_i^TA) 
\\&+ tr((A^T)^{(t)}A^{(t)}) +  tr(\frac{1}{\alpha_i}(A^T)^{(t)}d_id_i^TA) + tr(\frac{1}{\alpha_i}A^Td_id_i^TA^{(t)})
\\&+ tr(\frac{1}{\alpha_i^2}A^Td_id_i^Td_id_i^TA) 
\end{align*}

Equality \ref{eq:argmin}:
\begin{align*}  
  &\argmin \|A-A^{(t+1)}\|_F^2 
\\=& -2tr(\frac{1}{\alpha_i}A^Td_id_i^TA) +2tr(\frac{1}{\alpha_i}(A^T)^{(t)}d_id_i^TA) + tr(\frac{1}{\alpha_i^2}A^Td_id_i^Td_id_i^TA) 
\\=& \argmin_i \frac{2}{\alpha_i}tr((A^T)^{(t)}d_id_i^TA) -  \frac{2}{\alpha_i}tr(A^Td_id_i^TA) + \frac{1}{\alpha_i^2}\|d_id_i^TA\|_F^2 
\\=& \argmin_i \frac{2}{\alpha_i}\tilde x_i^T  x_i -  \frac{2}{\alpha_i}\|x_i\|_2^2 + \frac{1}{\alpha_i^2}\|d_i\|_2^2\|x_i\|_2^2 
\\=& \argmin_i \frac{2}{\alpha_i}(\tilde x_i-x_i)^Tx_i + \frac{1}{\alpha_i^2}\|x_i\|_2^2(\tilde x_{ii} - x_{ii}) 
\\=& \argmin_i \frac{2}{\alpha_i}\tilde x_i^Tx_i +    (\frac{-2}{\alpha_i} +  \frac{1}{\alpha_i^2}(\tilde x_{ii} - x_{ii})  )x_i^Tx_i + \lambda
\end{align*}

Equality \ref{eq:update_atat}
\begin{align*}
A^TA^{(t+1)} =& A^T (A^{(t)} + \frac{1}{\alpha_w} d_wd_w^TA)
\\ =& A^TA^{(t)} + \frac{1}{\alpha_w} A^Td_wd_w^TA
\\ =& A^TA^{(t)} + \frac{1}{\alpha_w}  x_w x_w^T
\end{align*}

Equality \ref{eq:update_ata}
\begin{align*}
(A^TA)^{(t+1)} =& (A^{(t)} + \frac{1}{\alpha_w} d_wd_w^TA  )^T (A^{(t)} + \frac{1}{\alpha_w} d_wd_w^TA)  
\\ =& (A^TA)^{(t)} + \frac{1}{\alpha_w} (A^T)^{(t)}d_wd_w^TA + \frac{1}{\alpha_w} A^Td_wd_w^TA^{(t)} + \frac{1}{\alpha_w^2} A^Td_wd_w^Td_wd_w^TA  
\\ =& (A^TA)^{(t)} + \frac{1}{\alpha_w} \tilde x_w x_w^T + \frac{1}{\alpha_w} x_w\tilde x_w^T + \frac{1}{\alpha_w^2} x_wd_w^Td_wx_w^T
\end{align*}

Equality \ref{eq:update_x}
\begin{align*}
 X^{(t+1)} & = A^TD^{(t+1)}  
\\ & = A^TA^{(t+1)}-A^TA 
\\ & = A^TA^{(t)} + \frac{1}{\alpha_w} x_w x_w^T -A^TA  
\\ & = A^TD^{(t)} + A^TA + \frac{1}{\alpha_w} x_w x_w^T -A^TA  
\\ & = X^{(t)} + \frac{1}{\alpha_w} x_w x_w^T = X^{(0)} + \sum_{i=0}^t \left (\frac{1}{a_w}x_wx_w^T \right )^{(i)}
\end{align*}

Equality \ref{eq:update_xt}
\begin{align*}
\tilde X^{(t+1)} & = (A^TD)^{(t+1)}  
\\ & = (A^TA)^{(t+1)}-(A^T)^{(t+1)}A 
\\ & = (A^TA)^{(t)} + \frac{1}{\alpha_w} \tilde x_w x_w^T + \frac{1}{\alpha_w} x_w\tilde x_w^T + \frac{1}{\alpha_w^2} x_wd_w^Td_wx_w^T - (A^T)^{(t)}A - \frac{1}{\alpha_w} x_w  x_w^T 
\\ & = (A^TA)^{(t)}  - (A^T)^{(t)}A + \frac{1}{\alpha_w} \tilde x_w x_w^T + \frac{1}{\alpha_w} x_w\tilde x_w^T +  \frac{1}{\alpha_w^2} x_wd_w^Td_wx_w^T - \frac{1}{\alpha_w} x_w  x_w^T 
\\ & = (A^TD)^{(t)} + \frac{1}{\alpha_w} \tilde x_w x_w^T + \frac{1}{\alpha_w} x_w\tilde x_w^T + \frac{1}{\alpha_w^2} x_wd_w^Td_wx_w^T  - \frac{1}{\alpha_w} x_w  x_w^T 
\\ & = X^{(0)} +  \sum_{i=0}^t \left (\frac{1}{\alpha_w} \tilde x_w x_w^T + \frac{1}{\alpha_w} x_w\tilde x_w^T + \frac{1}{\alpha_w^2} x_wd_w^Td_wx_w^T  - \frac{1}{\alpha_w} x_w  x_w^T \right )^{(i)}
\end{align*}




Declarations of interest: none.

\bibliographystyle{elsarticle-num}
\bibliography{citations}








\end{document}